%% file: paper.tex
\documentclass[acmsmall,screen]{acmart}%\settopmatter{printfolios=true,printccs=true,printacmref=true}
\setcopyright{rightsretained}
\acmPrice{}
\acmDOI{10.1145/3360544}
\acmYear{2019}
\copyrightyear{2019}
\acmJournal{PACMPL}
\acmVolume{3}
\acmNumber{OOPSLA}
\acmArticle{118}
\acmMonth{10}

\usepackage{booktabs}   %% For formal tables:
\usepackage{subcaption} %% For complex figures with subfigures/subcaptions
\usepackage{wrapfig}
\usepackage{caption}
\usepackage{enumitem}
\usepackage{amssymb}
\usepackage{multirow}

\usepackage{color}
\usepackage{fancyvrb}
\usepackage{tikz}
\usepackage{bussproofs}
\usepackage{algorithm}                                                                                                                 
\usepackage[noend]{algpseudocode}
\usepackage{xspace}

\usepackage{amsmath,amssymb}
\usepackage{amsthm}
\usepackage{algorithm}
\usepackage{algpseudocode}
\usepackage{stmaryrd}
\usepackage{xspace}
\usepackage{multirow}

\newtheorem{theorem}{Theorem}[section]
\newtheorem{definition}[theorem]{Definition}

\newtheorem{lemma}[theorem]{Lemma}

\newcommand{\V}{\mathcal{V}}
\newcommand{\Z}{\mathcal{Z}}

\newcommand{\R}{\mathcal{R}}
\newcommand{\X}{\mathcal{X}}
\newcommand{\M}{\mathcal{M}}

\newcommand{\vmin}{V_{\text{min}}}
\newcommand{\vmaj}{V_{\text{maj}}}
\newcommand{\rmin}{R_{\text{min}}}
\newcommand{\rmaj}{R_{\text{maj}}}
\newcommand{\mmin}{M_{\text{min}}}
\newcommand{\mmaj}{M_{\text{maj}}}

\newcommand{\true}{\textsf{true}\xspace}
\newcommand{\false}{\textsf{false}\xspace}

\newcommand{\toolname}{\textsc{VeriFair}\xspace}
\newcommand{\fairsquare}{\textsc{FairSquare}\xspace}

\begin{document}

%% Title information
\title{Probabilistic Verification of Fairness Properties via Concentration}
% \title[Short Title]{Full Title}         %% [Short Title] is optional;
%                                         %% when present, will be used in
%                                         %% header instead of Full Title.
% \titlenote{with title note}             %% \titlenote is optional;
%                                         %% can be repeated if necessary;
%                                         %% contents suppressed with 'anonymous'
% \subtitle{Subtitle}                     %% \subtitle is optional
% \subtitlenote{with subtitle note}       %% \subtitlenote is optional;
%                                         %% can be repeated if necessary;
%                                         %% contents suppressed with 'anonymous'

%% Author information
%% Contents and number of authors suppressed with 'anonymous'.
%% Each author should be introduced by \author, followed by
%% \authornote (optional), \orcid (optional), \affiliation, and
%% \email.
%% An author may have multiple affiliations and/or emails; repeat the
%% appropriate command.
%% Many elements are not rendered, but should be provided for metadata
%% extraction tools.

%% Author with single affiliation.
\author{Osbert Bastani}
%\authornote{with author1 note}          %% \authornote is optional;
                                        %% can be repeated if necessary
%\orcid{nnnn-nnnn-nnnn-nnnn}             %% \orcid is optional
\affiliation{
%  \position{Position1}
%  \department{Department1}              %% \department is recommended
  \institution{University of Pennsylvania}            %% \institution is required
%  \streetaddress{Street1 Address1}
%  \city{City1}
%  \state{State1}
%  \postcode{Post-Code1}
  \country{USA}                    %% \country is recommended
}
\email{obastani@seas.upenn.edu}          %% \email is recommended

%% Author with two affiliations and emails.
\author{Xin Zhang}
%\authornote{with author2 note}          %% \authornote is optional;
                                        %% can be repeated if necessary
%\orcid{nnnn-nnnn-nnnn-nnnn}             %% \orcid is optional
\affiliation{
%  \position{Position2a}
%  \department{Department2a}             %% \department is recommended
  \institution{MIT}           %% \institution is required
%  \streetaddress{Street2a Address2a}
%  \city{City2a}
%  \state{State2a}
%  \postcode{Post-Code2a}
  \country{USA}                   %% \country is recommended
}
\email{xzhang@csail.mit.edu}         %% \email is recommended

\author{Armando Solar-Lezama}
\affiliation{
%  \position{Position2b}
%  \department{Department2b}             %% \department is recommended
  \institution{MIT}           %% \institution is required
%  \streetaddress{Street3b Address2b}
%  \city{City2b}
%  \state{State2b}
%  \postcode{Post-Code2b}
  \country{USA}                   %% \country is recommended
}
\email{asolar@csail.mit.edu}         %% \email is recommended

%% Abstract
%% Note: \begin{abstract}...\end{abstract} environment must come
%% before \maketitle command
\begin{abstract}
As machine learning systems are increasingly used to make real world legal and financial decisions, it is of paramount importance that we develop algorithms to verify that these systems do not discriminate against minorities. We design a scalable algorithm for verifying fairness specifications. Our algorithm obtains strong correctness guarantees based on adaptive concentration inequalities; such inequalities enable our algorithm to adaptively take samples until it has enough data to make a decision. We implement our algorithm in a tool called \toolname, and show that it scales to large machine learning models, including a deep recurrent neural network that is more than five orders of magnitude larger than the largest previously-verified neural network. While our technique only gives probabilistic guarantees due to the use of random samples, we show that we can choose the probability of error to be extremely small.
\end{abstract}

%% 2012 ACM Computing Classification System (CSS) concepts
%% Generate at 'http://dl.acm.org/ccs/ccs.cfm'.

%% End of generated code

%% Keywords
%% comma separated list
  %% \keywords are mandatory in final camera-ready submission

%% \maketitle
%% Note: \maketitle command must come after title commands, author
%% commands, abstract environment, Computing Classification System
%% environment and commands, and keywords command.
\maketitle

\input{main}

% %% Acknowledgments
\begin{acks}                            %% acks environment is optional
%                                         %% contents suppressed with 'anonymous'
%   %% Commands \grantsponsor{<sponsorID>}{<name>}{<url>} and
%   %% \grantnum[<url>]{<sponsorID>}{<number>} should be used to
%   %% acknowledge financial support and will be used by metadata
%   %% extraction tools.
%   This material is based upon work supported by the
%   \grantsponsor{GS100000001}{National Science
%     Foundation}{http://dx.doi.org/10.13039/100000001} under Grant
%   No.~\grantnum{GS100000001}{nnnnnnn} and Grant
%   No.~\grantnum{GS100000001}{mmmmmmm}.  Any opinions, findings, and
%   conclusions or recommendations expressed in this material are those
%   of the author and do not necessarily reflect the views of the
%   National Science Foundation.
This work was supported by ONR N00014-17-1-2699.
\end{acks}

%% Bibliography
\bibliography{paper}

\clearpage

%% Appendix
\appendix
\input{appendix}

\end{document}

%% file: main.tex
\section{Introduction}

Machine learning is increasingly being used to inform sensitive decisions, including legal decisions such as whether to offer bail to a defendant~\cite{lakkaraju2017selective}, and financial decisions such as whether to give a loan to an applicant~\cite{hardt2016equality}. In these settings, for both ethical and legal reasons, it is of paramount importance that decisions are made fairly and without discrimination~\cite{zarsky2014understanding,barocas2016big}. Indeed, one of the motivations for introducing machine learning in these settings is the expectation that machines would not be subject to the same implicit biases that may affect human decision makers. However, designing machine learning models that satisfy fairness criterion has proven to be quite challenging, since these models have a tendency to internalize biases present in the data. Even if sensitive features such as race and gender are withheld from the model, it often internally reconstructs sensitive features.

Our goal is to verify whether a given fairness specification holds for a given machine learning model, focusing on specifications that have been proposed in the machine learning literature. In particular, our goal is \emph{not} to devise new specifications. There has been previous work on trying to verify probabilistic specifications~\cite{sankaranarayanan2013static,sampson2014expressing,gehr2016psi}, including work specifically targeting fairness~\cite{albarghouthi2017fairsquare}. Approaches based on symbolic integration~\cite{gehr2016psi} and numerical integration~\cite{albarghouthi2017fairsquare} have been proposed. However, these approaches can be extremely slow---indeed, previous work using numerical integration to verify fairness properties only scales to neural networks with a single hidden layer containing just three hidden units~\cite{albarghouthi2017fairsquare}, whereas state-of-the-art neural networks can have dozens of layers and millions of hidden units. There has also been prior work aiming to verify probabilistic specifications using approximate techniques such as belief propagation and sampling~\cite{sampson2014expressing}. While these techniques are much more scalable, they typically cannot give soundness guarantees; thus, they can be useful for bug-finding, but are not suitable for verifying fairness properties, where the ability to guarantee the fairness of a given model is very important.

Our approach is to do probabilistic verification by leveraging sampling, using concentration inequalities to provide strong soundness guarantees.
\footnote{We discuss limitations of our approach in Section~\ref{sec:discussion}. Furthermore, while our approach is not a priori specific to fairness, there are several challenges to applying it more broadly, which we also discuss in Section~\ref{sec:discussion}.}
In particular, we provide guarantees of the form:
%Our approach is to leverage sampling, but using \emph{concentration inequalities} that can provide soundness guarantees that are essentially as strong as those provided by numerical integration. In particular, we provide guarantees of the form
\begin{align}
\label{eqn:intro}
\text{Pr}[\hat{Y}=Y]\ge1-\Delta,
\end{align}
where $\hat{Y}$ is the response provided by our algorithm (i.e., whether the specification holds for the given model), and $Y$ is the true answer. To enable such guarantees, we rely on \emph{adaptive concentration inequalities}~\cite{zhao2016adaptive}, which are concentration inequalities where our verification algorithm can improve its estimate $\hat{Y}$ of $Y$ until Eq. \ref{eqn:intro} holds.

We prove that our verification algorithm is both sound and precise in this high-probability sense. While in principle the probabilistic guarantee expressed by the formula is weaker than a traditional soundness guarantee, we show that our approach allows us to efficiently prove the above property with $\Delta$ very close to zero. For example, in our evaluation on a deep neural network, we take $\Delta=10^{-10}$, meaning there is only a $10^{-10}$ probability that a program that our algorithm verifies to be fair is actually unfair. In contrast, the probability of winning the October 2018 Powerball was about 3 times higher (roughly $3\times10^{-9}$)~\cite{powerball}. To the best of our knowledge, our work is the first to use adaptive concentration inequalities to design a probabilistically sound and precise verification algorithm.
%We show that we can take $\Delta$ to be tiny (in our evaluation on a deep neural network, as small as $\Delta=10^{-10}$) and our algorithm still terminates quickly.

Furthermore, while our algorithm is incomplete and can fail to terminate on certain problem instances, we show that nontermination can only occur under an unlikely condition. Intuitively, nontermination can happen in cases where a specification ``just barely holds''---i.e., for a random variable $X$, we want to show that $\mathbb{E}[X]\ge0$, but $\mathbb{E}[X]=0$. Then, the error in our estimate of $\mathbb{E}[X]$ will never be small enough to determine whether $\mathbb{E}[X]\ge0$ holds. Except in these cases, we prove that our algorithm terminates in finite time with probability $1$.

We implement our algorithm in a tool called \toolname, which can be used to verify fairness properties of programs.
\footnote{\toolname is available at \url{https://github.com/obastani/verifair}.}
In particular, we compare \toolname to the state-of-the-art fairness verification tool \fairsquare~\cite{albarghouthi2017fairsquare}; our tool outperforms theirs on each of the 12 largest problem instances in their benchmark. Furthermore, the \fairsquare benchmarks are implemented in Python; compiling their problem instances to native code can yield more than a 200$\times$ increase in the performance of \toolname (in contrast, the running time of their tool is not increased this way, since they use symbolic techniques). Finally, we evaluate \toolname on a much larger benchmark: we study a deep neural network used to classify human-drawn sketches of various objects~\cite{drawclassify,ha2017neural}. Our benchmark consists of neural networks with about 16 million parameters, which is more than 5 orders of magnitude larger than the largest neural network in the \fairsquare benchmark, which has 37 parameters. On this benchmark, \toolname terminates in just 697 seconds (with probability of error $\Delta=10^{-10}$). This result shows that \toolname can scale to large image classification tasks, for which fairness is often an important property---for example, login systems based on face recognition have been shown to make more mistakes detecting minority users than detecting majority users~\cite{simon2009hp}. In summary, our contributions are
\begin{itemize}
\item We propose an algorithm for verifying fairness properties of machine learning models based on adaptive concentration inequalities (Section~\ref{sec:algorithm}).
\item We prove that our algorithm is sound and precise in a high-probability sense, and guarantee termination except in certain pathelogical cases---most importantly, fairness does not ``just barely'' hold (Section~\ref{sec:theory}).
\item We implement our algorithm in a tool called \toolname. We show that \toolname substantially outperforms the state-of-the-art fairness verifier \fairsquare, and can furthermore scale to problem instances more than $10^6\times$ larger than \fairsquare (Section~\ref{sec:evaluation}).
\end{itemize}

\section{Motivating Example}

Consider the simple classifier $f_{\text{job}}$ shown on the left-hand side of Figure~\ref{fig:example} (adapted from~\cite{albarghouthi2017fairsquare}). This classifier predicts whether a given candidate should be offered a job based on two features: the ranking of the college they attended and their number of years of work experience. For both legal and ethical reasons, we may want to ensure that $f_{\text{job}}$ does not discriminate against minorities. There are a number of ways to formalize nondiscrimination. In this section, we show how our techniques can be applied to checking a fairness specification called \emph{demographic parity}~\cite{calders2009building}; we discuss additional fairness specifications of interest in Section~\ref{sec:fairnessspecifications}. Demographic parity is based on legal guideline for avoiding hiring discrimination is the ``80\% rule''~\cite{biddle2006adverse}. This rule says that the rate at which minority candidates are offered jobs should be at least 80\% of the rate at which majority candidates are offered jobs:
\begin{align*}
Y_{\text{job}}\equiv\left(\frac{\mu_{\text{female}}}{\mu_{\text{male}}}\ge0.8\right),
\end{align*}
where
\begin{align*}
\mu_{\text{male}}&=\text{Pr}[\text{offer}=1\mid\text{gender}=\text{male}] \\
\mu_{\text{female}}&=\text{Pr}[\text{offer}=1\mid\text{gender}=\text{female}].
\end{align*}
Then, $f_{\text{job}}$ satisfies demographic parity if $Y_{\text{job}}=\true$.

\begin{figure}
\footnotesize
\begin{verbatim}
        def offer_job(col_rank, years_exp)         def population_model():
           if col_rank <= 5:                          is_male ~ bernoulli(0.5)
              return true                             col_rank ~ normal(25, 10)
           elif years_exp > 5:                        if is_male:
              return true                                years_exp ~ normal(15, 5)
           else:                                      else:
              return false                               years_exp ~ normal(10, 5)
                                                      return col_rank, years_exp
\end{verbatim}
\caption{Left: A classifier $f_{\text{job}}:\mathbb{R}^2\to\{\true,\false\}$ for deciding whether to offer a job to a candidate (adapted from~\cite{albarghouthi2017fairsquare}). This classifier takes as input two features---the candidate's college ranking (${\tt col\_rank}$), and the candidate's years of work experience (${\tt years\_exp}$). Right: A population model $P_{\text{job}}$ over the features ${\tt is\_male}$, ${\tt col\_rank}$, and ${\tt years\_exp}$ of job candidates. Note that a candidate's years of work experience is affected by their gender.}
\label{fig:example}
\end{figure}

Note that the demographic parity specification assumes given a distribution $P$ of features for job candidates, which we call a \emph{population model}~\cite{albarghouthi2017fairsquare}, since $\mu_{\text{male}}$ and $\mu_{\text{female}}$ are conditional expectations over this distribution. In general, a population model is specified as a probabilistic program that takes no inputs, and returns the features (i.e., college ranking and years of work experience) for a randomly sampled member of that population. For example, on the right-hand side of Figure~\ref{fig:example}, we show a population model $P_{\text{job}}$ over job candidates. We refer to $P_{\text{job}}\mid\text{gender}=\text{male}$ as the \emph{majority subpopulation}, and $P_{\text{job}}\mid\text{gender}=\text{female}$ as the \emph{minority subpopulation}. In this example, male candidates have more years of experience on average than female candidates, but they have the same college ranking on average. We discuss how population models can be obtained in Section~\ref{sec:discussion}.

Given classifier $f_{\text{job}}$, demographic parity specification $Y_{\text{job}}$ with population model $P_{\text{job}}$, and a desired confidence level $\Delta\in\mathbb{R}_+$, the goal of our verification algorithm is to check whether $Y_{\text{job}}$ holds. In particular, our algorithm estimates the fairness of $f$ by iteratively sampling random values
\begin{align*}
V_{a,1},...,V_{a,n}\sim P_{\text{job}}\mid\text{gender}=a
\end{align*}
for each $a\in\{\text{male},\text{female}\}$, and then using these samples to estimate $\mu_{\text{male}}$ and $\mu_{\text{female}}$:
\begin{align*}
\hat{\mu}_a=\frac{1}{n}\sum_{i=1}^nf(V_{a,i}).
\end{align*}
Then, our algorithm uses $\hat{\mu}_{\text{male}}$ and $\hat{\mu}_{\text{female}}$ to estimate $Y_{\text{job}}$:
\begin{align*}
\hat{Y}_{\text{job}}\equiv\left(\frac{\hat{\mu}_{\text{female}}}{\hat{\mu}_{\text{male}}}\ge0.8\right).
\end{align*}
Note that $\hat{Y}_{\text{job}}$ is easy to compute; the difficulty is bounding the probability of error, namely, $\gamma=\text{Pr}[\hat{Y}_{\text{job}}\neq Y_{\text{job}}]\in\mathbb{R}_+$. In particular, our estimates $\hat{\mu}_{\text{male}}$ and $\hat{\mu}_{\text{female}}$ may have errors; thus, $\hat{Y}_{\text{job}}$ may differ from the true value $Y_{\text{job}}$. It is well known that $\gamma\to0$ as the number of samples $n$ goes to infinity; thus, while we can never guarantee that fairness holds, we can do so with arbitrarily high confidence. In particular, for any $\Delta\in\mathbb{R}_+$, our algorithm returns $\hat{Y}_{\text{job}}$ satisfying
\begin{align}
\label{eqn:guarantee}
\text{Pr}[\hat{Y}_{\text{job}}=Y_{\text{job}}]\ge1-\Delta.
\end{align}
The key challenge is establishing finite sample bounds on $\gamma$, and furthermore, doing so in an adaptive way so it can collect as much data as needed to ensure that Eq. \ref{eqn:guarantee} holds (i.e., $\gamma\le\Delta$). In particular, there are two key techniques our algorithm uses to establish Eq. \ref{eqn:guarantee}. First, our algorithm uses an \emph{adaptive concentration inequality} (from~\cite{zhao2016adaptive}) to establish lemmas on the error of the estimates $\hat{\mu}_{\text{male}}$ and $\hat{\mu}_{\text{female}}$, e.g.,
\begin{align}
\label{eqn:lemma}
\text{Pr}[|\hat{\mu}_a-\mu_a|\le\varepsilon]&\ge1-\delta_a
\end{align}
for $a\in\{\text{male},\text{female}\}$. Standard concentration inequalities can only establish bounds of the form Eq.~\ref{eqn:lemma} for a fixed number of samples $n$. However, our algorithm cannot a priori know how many samples it needs to establish Eq.~\ref{eqn:guarantee}; instead, it adaptively takes new samples until it determines that Eq.~\ref{eqn:guarantee} holds. To enable this approach, we use adaptive concentration inequalities, which we describe in Section~\ref{sec:adaptive}.

Second, it uses the lemmas in Eq.~\ref{eqn:lemma} to derive a bound
\begin{align*}
\text{Pr}[\hat{Y}_{\text{job}}=Y_{\text{job}}]\ge1-\gamma.
\end{align*}
We describe how our algorithm does so in Section~\ref{sec:inference}.

Finally, our algorithm terminates once $\gamma\le\Delta$, at which point we guarantee that the estimate $\hat{Y}_{\text{job}}$ satisfies Eq.~\ref{eqn:guarantee}, i.e., our algorithm accurately outputs whether fairness holds with high probability. In particular, our algorithm is sound and precise in a probabilistic sense. Furthermore, our algorithm terminates with probability $1$ unless the problem instance is pathelogical in one of two ways: (i) $\mu_{\text{male}}=0$ (so $Y_{\text{job}}$ contains a division by zero), or (ii) fairness ``just barely'' holds, i.e., $\frac{\mu_{\text{female}}}{\mu_{\text{male}}}=0.8$. In our evaluation, we show that even for $\Delta=10^{-10}$, our algorithm terminates quickly on a deep neural network benchmark---i.e., we can feasibly require that our algorithm make a mistake with probability at most $10^{-10}$.

\begin{figure}
\small
\begin{minipage}{0.15\textwidth}
\begin{alignat*}{3}
&T::=~&&\mu_Z\mid...&& \\
&&&\mid c\mid...&& \\
&&&\mid T+T&& \\
&&&\mid -T&& \\
&&&\mid T\cdot T&& \\
&&&\mid T^{-1}&& \\
&S::=~&&T\ge0&& \\
&&&\mid S\wedge S&& \\
&&&\mid S\vee S&& \\
&&&\mid\neg S.&&
\end{alignat*}
\end{minipage}
\hspace{1.5in}
\begin{minipage}{0.25\textwidth}
\begin{align*}
\llbracket\mu_Z\rrbracket&=\mu_Z \\
\llbracket c\rrbracket&=c \\
\llbracket X+X'\rrbracket&=\llbracket X\rrbracket+\llbracket X'\rrbracket \\
\llbracket -X\rrbracket&=-\llbracket X\rrbracket \\
\llbracket X\cdot X'\rrbracket&=\llbracket X\rrbracket\cdot\llbracket X'\rrbracket \\
\llbracket X^{-1}\rrbracket&=\llbracket X\rrbracket^{-1} \\
\llbracket X\ge0\rrbracket&=\mathbb{I}[\llbracket X\rrbracket\ge0] \\
\llbracket Y\wedge Y'\rrbracket&=\llbracket Y\rrbracket\wedge\llbracket Y'\rrbracket \\
\llbracket Y\vee Y'\rrbracket&=\llbracket Y\rrbracket\vee\llbracket Y'\rrbracket \\
\llbracket\neg Y\rrbracket&=\neg\llbracket Y\rrbracket.
\end{align*}
\end{minipage}
\caption{Left: Specification syntax. Here, $S$ and $T$ are nonterminal symbols (with $S$ being the start symbol), and the remaining symbols are terminals. The terminal symbols $\mu_Z,...$ represent the respective means of given Bernoulli random variables $Z,...$. In our setting, $Z,...$ typically encode the distribution of some statistic (e.g., rate of positive decisions) of $f$ for some subpopulation. The terminal symbols $c,...\in\mathbb{R}$ represent real-valued constants. Right: Specification semantics. Here, $X\in\mathcal{L}(T)$ and $Y\in\mathcal{L}(S)$ (where $\mathcal{L}(A)$ is the context-free language generated by $A$). The indicator function $\mathbb{I}[C]$ returns \true if $C$ holds and \false otherwise.}
\label{fig:specificationlanguage}
\end{figure}

\section{Problem Formulation}
\label{sec:problem}

We formalize the fairness properties that our algorithm can verify; our formulation is based on previous work~\cite{albarghouthi2017fairsquare}.

\subsection{Verification Algorithm Inputs}

\paragraph{{\bf\em Classification program.}}

Our goal is to verify fairness properties for a deterministic program $f:\V\to\{0,1\}$ that maps given members of a population $\V$ (e.g., job applicants) to a single binary output $r\in\R=\{0,1\}$ (e.g., whether to offer the applicant a job). For example, $f$ may be a machine learning classifier such as a neural network. Note that $f$ may use parameters learned from training data; in this case, our verification algorithm operates on the output of the training algorithm. Our verification algorithm only requires blackbox access to $f$, i.e., for any chosen input $v\in\V$, it can execute $f$ on $v$ to obtain the corresponding output $r=f(v)$.

\paragraph{{\bf\em Population model.}}

We assume we are given a probability distribution $P_{\V}$ over $\V$, which we refer to as the \emph{population model}, encoded as a probabilistic program that takes no inputs and construct a random member $V\sim P_{\V}$ of the population. Furthermore, we assume that our algorithm can sample conditional distributions $P_{\V}\mid C$, for some logical predicate $C$ over $\V$ (i.e., $C:\V\to\{\true,\false\}$). For example, assuming $\V$ is discrete, our algorithm can do so using rejection sampling---we randomly sample $V\sim P_{\V}$ until $C(V)=\true$, and return this $V$. The predicate $C$ is dependent on the fairness property that we are trying to prove; in our evaluation, we show that for the fairness properties that we study, the necessary predicates have sufficiently large support that rejection sampling is reasonably efficient.

\paragraph{{\bf\em Specification language.}}

The syntax and semantics of the specifications that we aim to verify are shown in Figure~\ref{fig:specificationlanguage}. The start symbol of the grammar is $S$. In this grammar, the symbol $\mu_Z$ (where $Z$ is a Bernoulli random variable) represents the expected value of $Z$, and $c\in\mathbb{R}$ is a numerical constant. The remainder of this grammar enables us to construct arithmetic expressions of the expectated values $\mu_Z$ and the constants $c$. Intuitively, this specification language enables us to encode arithmetic relationships between various conditional expectations that should hold. The advantage of introducing a specification language is that we can flexibly verify a wide range of fairness specifications in the same framework. As we show in Section~\ref{sec:fairnessspecifications}, a number of fairness specifications that have been proposed in the literature can be expressed in our specification language.

\subsection{Fairness Specifications}
\label{sec:fairnessspecifications}

Next, we describe how three fairness specifications from the machine learning literature can be formalized in our specification language; the best fairness specification to use is often context specific. We discuss additional specifications that can be represented in our language in Section~\ref{sec:discussion}. We first establish some notation. For any probability distribution $P_{\Z}$ over a space $\Z$ with corresponding random variable $Z\sim P_{\Z}$, we let $\mu_Z=\mathbb{E}_{Z\sim P_{\Z}}[Z]$ denote the expectation of $Z$. Recall that for a Bernoulli random variable $Z\sim P_{\Z}$, we have $\mu_Z=\text{Pr}_{Z\sim P_{\Z}}[Z=1]$.

\paragraph{{\bf\em Demographic parity.}}

Intuitively, our first property says that minority members should be classified as $f(V)=1$ at approximately the same rate as majority candidates~\cite{calders2009building}.
\begin{definition}
\label{def:demographicparity}
Let
\begin{align*}
  \vmaj&\sim P_{\V}\mid A=\text{maj} \\
  \vmin&\sim P_{\V}\mid A=\text{min}
\end{align*}
be conditional random variables for members of the majority and minority subpopulations, respectively. Let $\rmaj=f(\vmaj)$ and $\rmin=f(\vmin)$ be the Bernoulli random variables denoting whether the classifier $f$ offers a favorable outcome to a member of the majority and minority subpopulation, respectively. Given $c\in[0,1]$, the {\bf\em demographic parity} property is
\begin{align*}
  Y_{\text{parity}}\equiv\left(\frac{\mu_{\rmin}}{\mu_{\rmaj}}\ge1-c\right).
\end{align*}
\end{definition}
In our example of hiring, the majority subpopulation is $P_{\text{job}}\mid\text{gender}=\text{male}$, the minority subpopulation is $P_{\text{job}}\mid\text{gender}=\text{female}$, and the classifier $f_{\text{job}}:\mathbb{R}^2\to\{0,1\}$ determines whether a candidate with the given years of experience and college rank is offered a job. Then, demographic parity says that for every male candidate offered a job, at least $1-c$ female candidates should be offered a job.

\paragraph{{\bf\em Equal opportunity.}}

Intuitively, our second property says that \emph{qualified} members of the minority subpopulation should be classified as $f(V)=1$ at roughly the same rate as qualified members of the majority subpopulation~\cite{hardt2016equality}.
\begin{definition}
\label{def:equalopportunity}
Let $q\in\mathcal{Q}=\{\text{qual},\text{unqual}\}$ indicate whether the candidate is qualified, and let
\begin{align*}
  \vmaj&\sim P_{\V}\mid A=\text{maj},~Q=\text{qual} \\
  \vmin&\sim P_{\V}\mid A=\text{min},~Q=\text{qual}
\end{align*}
be conditional random variables over $\V$ representing qualified members of the majority and minority subpopulations, respectively. Let $\rmaj=f(\vmaj)$ and $\rmin=f(\vmin)$ denote whether candidates $\vmaj$ and $\vmin$ are offered jobs according to $f$, respectively. Then, the {\bf\em equal opportunity} property is
\begin{align*}
  Y_{\text{equal}}\equiv\left(\frac{\mu_{\rmin}}{\mu_{\rmaj}}\ge1-c\right)
\end{align*}
for a given constant $c\in[0,1]$.
\end{definition}
Continuing our example, this property says that for every job offered to a \emph{qualified} male candidate, at least $1-c$ \emph{qualified} female candidates should be offered a job as well.

\paragraph{{\bf\em Path-specific causal fairness.}}

Intuitively, our third property says that the outcome (e.g., job offer) should not depend directly on a sensitive variable (e.g., gender), but may depend indirectly on the sensitive covariate through other \emph{mediator covariates} deemed directly relevant to predicting job performance (e.g., college degree)~\cite{nabi2018fair}. For simplicity, we assume that the mediator covariate $\M=\{0,1\}$ is binary, that we are given a distribution $P_{\M}$ over $\M$, and that the classifier $f:\V\times\M\to\{0,1\}$ is extended to be a function of $\M$.
\begin{definition}
\label{def:pathspecificfairness}
Let
\begin{align*}
  \vmaj&\sim P_{\V}\mid A=\text{maj} \\
  \mmaj&\sim P_{\M}\mid A=\text{maj},~V=\vmaj \\
  \rmaj&=f(\vmaj,\mmaj)
\end{align*}
be how a member of the majority subpopulation is classified by $f$, and let
\begin{align*}
  \vmin&\sim P_{\V}\mid A=\text{min} \\
  \mmin&\sim P_{\M}\mid A=\text{maj},~V=\vmin \\
  \rmin&=f(\vmin,\mmin)
\end{align*}
be how a member of the minority subpopulation is classified by $f$, except that their mediator covariate $M$ is drawn as if they were a member of the majority subpopulation. Given $c\in[0,1]$, the {\bf\em path-specific causal fairness} property is
\begin{align*}
  Y_{\text{causal}}\equiv(\mu_{\rmin}-\mu_{\rmaj}\ge-c).
\end{align*}
\end{definition}
The key insight in this specification is how we sample the mediator variable $\mmin$ for a member $\vmin$ of the minority population. In particular, we sample $\mmin$ conditioned on the characteristics $\vmin$, except that we change the sensitive attribute to $A=\text{maj}$ instead of $A=\text{min}$. Intuitively, $\mmin$ is the value of the mediator variable if $\vmin$ were instead a member of the majority population, but everything else about them stays the same. In our example, suppose that we have a mediator covariate $\text{college}$ (either yes or no) and a non-mediator covariate $\texttt{years\_exp}$.
%and the goal is to allow $f$ to discriminate based on years of work experience, but not on whether the candidate went to college.
Then, the path-specific causal fairness property says that a female candidate should be given a job offer with similar probability as a male candidate---except she went to college as if she were a male candidate (but everything else about her---i.e., her years of job experience---stays the same). Thus, this specification measures the effect of gender on job offer, but ignoring the effect of gender on whether they went to college.

\section{Verification Algorithm}
\label{sec:algorithm}

We now describe our verification algorithm.

\begin{algorithm}[t]
\begin{algorithmic}
\Procedure{Verify}{$P_{\Z},Y,\Delta$}
\State $s\gets0$
\State $n\gets0$
\While{{\bf true}}
\State $Z\sim P_{\Z}$
\State $s\gets s+Z$
\State $n\gets n+1$
\State $\delta_Z\gets\Delta/\llbracket Y\rrbracket_{\delta}$
\State $\varepsilon_Z\gets\varepsilon(\delta_Z,n)$
\State $\Gamma\gets\left\{\mu_Z:\left(s/n,\varepsilon_Z,\delta_Z\right)\right\}$
\If{$\Gamma\vdash Y:(\true,\gamma)$ {\bf and} $\gamma\le\Delta$}
\State\Return{\bf \true}
\ElsIf{$\Gamma\vdash Y:(\false,\gamma)$ {\bf and} $\gamma\le\Delta$}
\State\Return{\bf \false}
\EndIf
\EndWhile
\EndProcedure
\end{algorithmic}
\caption{Algorithm for verifying the given specification $Y\in\mathcal{L}(S)$. The quantity $\varepsilon(\delta_Z,n)$ is defined in Eq.~\ref{eqn:epsilon}. 
The rules for checking $\Gamma\vdash Y:(I,\gamma)$, for $I\in\{\true,\false\}$, are shown in Figure~\ref{fig:inference}.}
\label{alg:verify}
\end{algorithm}

\subsection{High-Level Algorithm}

The intuition behind our algorithm is that for a Bernoulli random variable $Z$ with distribution $P_{\Z}$, we can use a fixed number of random samples $Z_1,...,Z_n\sim P_{\Z}$ to estimate $\mu_Z$:
\begin{align}
\label{eqn:bernoulliestimate}
\hat{\mu}_Z=\frac{1}{n}\sum_{i=1}^nZ_i.
\end{align}
Note that no matter how many samples we take, there may always be some error $\varepsilon$ between our estimate $\hat{\mu}_Z$ and the true expected value $\mu_Z$. Our algorithm uses adaptive concentration inequalities to prove high-probability bounds on this error. Then, it uses these bounds to establish high-probability bounds on the output of our algorithm---i.e., whether the fairness specification holds. We describe each of these components in more detail in the remainder of this section.

\paragraph{{\bf\em Adaptive concentration inequalities.}}

We can use concentration inequalities to establish high-probability bounds on the error $|\hat{\mu}_Z-\mu_Z|$ of our estimate $\hat{\mu}_Z$ of $\mu_Z$ of the form
\begin{align}
\label{eqn:bernoulliconcentration}
\text{Pr}_{Z_1,...,Z_n\sim P_{\Z}}[|\hat{\mu}_Z-\mu_Z|\le\varepsilon]\ge1-\delta.
\end{align}
Note that the probability is taken over the (independent) random samples $Z_1,...,Z_n\sim P_{\Z}$ used in the estimate $\hat{\mu}_Z$; when there is no ambiguity, we omit this notation.

Our algorithm uses adaptive concentration inequalities to establish bounds of the form Eq.~\ref{eqn:bernoulliconcentration}. In particular, they enable the algorithm to continue to take samples to improve its estimate $\hat{\mu}_Z$. Once our algorithm terminates, the adaptive concentration inequality guarantees that a bound of the form Eq.~\ref{eqn:bernoulliconcentration} holds (for a given $\delta\in\mathbb{R}_+$; then, $\varepsilon$ is a function of $\delta$ specified by the inequality). We describe the adaptive concentration inequalities we use in Section~\ref{sec:adaptive}.

\paragraph{{\bf\em Concentration for expressions.}}

Next, consider an expression $X\in\mathcal{L}(T)$. We can use substitute $\hat{\mu}_Z$ for $\mu_Z$ in $X$ to obtain an estimate $E$ for $\llbracket X\rrbracket$. Then, given that Eq.~\ref{eqn:bernoulliconcentration} holds, we show how to derive high-probability bounds of the form
\begin{align}
\label{eqn:xconcentration}
\text{Pr}[|E-\llbracket X\rrbracket|\le\varepsilon]\ge1-\delta.
\end{align}
We use the notation $X:(E,\varepsilon,\delta)$ to denote that Eq.~\ref{eqn:xconcentration} holds; we call this relationship a \emph{lemma}. Similarly, for $Y\in\mathcal{L}(S)$, we can substitute $\hat{\mu}_Z$ for $\mu_Z$ in $Y$ to obtain an estimate $I$ for $\llbracket Y\rrbracket$, and derive high-probability bounds of the form
\begin{align}
\label{eqn:yconcentration}
\text{Pr}[I=\llbracket Y\rrbracket]\ge1-\gamma.
\end{align}
Unlike Eq.~\ref{eqn:xconcentration}, we can establish that $I$ exactly equals $\llbracket Y\rrbracket$ with high probability; this difference arises because $\llbracket Y\rrbracket\in\{\true,\false\}$ are discrete values, whereas $\llbracket X\rrbracket\in\mathbb{R}$ are continuous values. We describe inference rules used to derive these lemmas $X:(E,\varepsilon,\delta)$ and $Y:(I,\gamma)$ in Section~\ref{sec:inference}.

\paragraph{{\bf\em Verification algorithm.}}

Given a classifier $f:\V\to\{0,1\}$, a population model $P_{\V}$, a specification $Y\in\mathcal{L}(S)$, and a confidence level $\Delta\in\mathbb{R}_+$, our goal is determine whether $Y$ is true with probability at least $1-\Delta$. For simplicity, we assume that $Y$ only has a single subexpression of the form $\mu_Z$ (where $Z$ is a Bernoulli random variable with distribution $P_{\Z}$); it is straightforward to generalize to the case where $Y$ contains multiple such subexpressions. At a high level, our algorithm iteratively computes more and more accurate estimates $\hat{\mu}_Z$ of $\mu_Z$ until $\hat{\mu}_Z$ is sufficiently accurate such that it can be used to compute an estimate $I$ of $\llbracket Y\rrbracket$ satisfying Eq.~\ref{eqn:yconcentration}. In particular, on the $n$th iteration, our algorithm performs these steps:
\begin{enumerate}
\item Draw a random sample $Z_n\sim P_{\Z}$, and update its estimate $\hat{\mu}_Z$ of $\mu_Z$ according to Eq.~\ref{eqn:bernoulliestimate}.
\item Establish a lemma $\mu_Z:(\hat{\mu}_Z,\varepsilon_Z,\delta_Z)$ using the adaptive concentration inequality (for a chosen value of $\delta_Z$).
\item Use the inferences rules to derive a lemma $Y:(I,\gamma)$ from the lemma in the previous step.
\item Terminate if $\gamma\le\Delta$; otherwise, continue.
\end{enumerate}
The full algorithm is shown in Algorithm~\ref{alg:verify}. In the body of the algorithm, $s$ is a running sum of the $n$ samples $Z_1,...,Z_n\sim P_{\Z}$ taken so far, so $\hat{\mu}_Z=\frac{s}{n}$. The variables $\delta_Z$ and $\varepsilon_Z$ come from our adaptive concentration inequality, described in Section~\ref{sec:adaptive}. Furthermore, $\delta_Z$ is chosen to be sufficiently small such that we can compute an estimate $I$ of $\llbracket Y\rrbracket$ with the desired confidence level $\Delta$, as we describe in Section~\ref{sec:choosingdelta}.

\subsection{Adaptive Concentration Inequalities}
\label{sec:adaptive}

Concentration inequalities can be used to establish bounds of the form Eq.~\ref{eqn:bernoulliconcentration}. For example, Hoeffding's inequality says Eq.~\ref{eqn:bernoulliconcentration} holds for $\delta=2e^{-2n\varepsilon^2}$ (equivalently, $\varepsilon=\sqrt{\frac{1}{2n}\log\frac{2}{\delta}}$)~\cite{hoeffding1963probability}:
\begin{align}
\label{eqn:hoeffding}
\text{Pr}_{Z_1,...,Z_n\sim P_{\Z}}[|\hat{\mu}_Z-\mu_Z|\le\varepsilon]\ge1-2e^{-2n\varepsilon^2}.
\end{align}
Then, for any $\varepsilon,\delta\in\mathbb{R}_+$, we can establish Eq.~\ref{eqn:bernoulliconcentration} by taking $n$ sufficiently large---in particular, because $2e^{-2n\varepsilon^2}\to0$ as $n\to\infty$, so for sufficiently large $n$, we have $\delta\le2e^{-2n\varepsilon^2}$.

A priori, we cannot know how large $n$ must be, since we do not know how small $\varepsilon$ must be for us to be able to prove or disprove the fairness specification. For example, for a specification of form $Y\equiv(\mu_Z\ge d)$, if $\mu_Z$ is very close to $d$, then we need $\varepsilon$ to be very small to ensure that our estimate $\hat{\mu}_Z$ is close to $\mu_Z$. For example, consider the two conditions
\begin{align}
\label{eqn:bootstrapbound}
&C_0:~\hat{\mu}_Z-d-\varepsilon\ge0 \\
&C_1:~\hat{\mu}_Z-d+\varepsilon<0 \nonumber
\end{align}
If $C_0$ holds, then together with the fact that $|\hat{\mu}_Z-\mu_Z|\le\varepsilon$, we can conclude that
\begin{align*}
\mu_Z\ge\hat{\mu}_Z-\varepsilon\ge d,
\end{align*}
Similarly, if $C_1$ holds, then we can conlude that
\begin{align*}
\mu_Z\le\hat{\mu}_Z+\varepsilon<d,
\end{align*}
However, a prior, we do not know $\hat{\mu}_Z$, so we cannot directly use these conditions to determine how small to take $\varepsilon$. Instead, our algorithm iteratively samples more and more points so $\varepsilon$ becomes smaller and smaller (for fixed $\delta$) until one of the two conditions $C_0$ and $C_1$ in Eq.~\ref{eqn:bootstrapbound} holds.

To implement this strategy, we have to account for multiple hypothesis testing. In particular, we need to establish a series of bounds for the estimates $\hat{\mu}_Z^{(0)},~\hat{\mu}_Z^{(1)},~...$ of $\mu_Z$ on successive iterations of our algorithm. For simplicitly, suppose that we apply Eq.~\ref{eqn:hoeffding} to two estimates $\hat{\mu}_Z^{(0)}$ and $\hat{\mu}_Z^{(1)}$ of $\mu_Z$:
\begin{align*}
\text{Pr}[|\hat{\mu}_Z^{(0)}-\mu_Z|\le\varepsilon]&\ge1-\delta \\
\text{Pr}[|\hat{\mu}_Z^{(1)}-\mu_Z|\le\varepsilon]&\ge1-\delta,
\end{align*}
where $\delta=2e^{-2n\varepsilon^2}$. The problem is that while we have established that each of the two events $|\hat{\mu}_Z^{(0)}-\mu_Z|\le\varepsilon$ and $|\hat{\mu}_Z^{(1)}-\mu_Z|\le\varepsilon$ occur with high probability $1-\delta$, we need for \emph{both} of these events to hold with high probability. One way we can do so is to take a union bound, in which case we get
\begin{align*}
\text{Pr}[|\hat{\mu}_Z^{(0)}-\mu_Z|\le\varepsilon\wedge|\hat{\mu}_Z^{(1)}-\mu_Z|\le\varepsilon]&\ge1-2\delta.
\end{align*}
Rather than building off of Hoeffding's inequality, our algorithm uses \emph{adaptive} concentration inequalities, which naturally account for multiple hypothesis testing. In particular, they enable our algorithm to continue to take samples to improve its estimate $\hat{\mu}_Z$. Upon termination, our algorithm has obtained $J$ samples $Z_i\sim P_{\Z}$. Note that $J$ is a random variable, since it depends on the previously taken samples $Z_i$, which our algorithm uses to decide when to terminate. Then, an adaptive concentration inequality guarantees that a bound of the form Eq.~\ref{eqn:bernoulliconcentration} holds, where $J$ is substituted for $n$ and $\varepsilon$ is specified by the bound. In particular, we use the following adaptive concentration inequality based on ~\cite{zhao2016adaptive}.
\begin{theorem}
\label{thm:bernoulliconcentration}
Given a Bernoulli random variable $Z$ with distribution $P_{\Z}$, let $\{Z_i\sim P_{\Z}\}_{i\in\mathbb{N}}$ be i.i.d. samples of $Z$, let
\begin{align*}
\hat{\mu}_Z^{(n)}=\frac{1}{n}\sum_{i=1}^nZ_i,
\end{align*}
let $J$ be a random variable on $\mathbb{N}\cup\{\infty\}$ such that $\text{Pr}[J<\infty]=1$, and let
\begin{align}
  \label{eqn:epsilon}
  \varepsilon(\delta,n)=\sqrt{\frac{\frac{3}{5}\cdot\log(\log_{11/10}n+1)+\frac{5}{9}\cdot\log(24/\delta)}{n}}.
\end{align}
Then, given any $\delta\in\mathbb{R}_+$, we have
\begin{align*}
\text{Pr}[|\hat{\mu}_Z^{(J)}-\mu_Z|\le\varepsilon(\delta,J)]\ge1-\delta.
\end{align*}
\end{theorem}
We give a proof in Appendix~\ref{sec:bernoulliconcentrationproof}.
%One subtlety in Theorem~\ref{thm:bernoulliconcentration} is that it only holds if we can guarantee that $J$ is finite with probability $1$---or equivalently, that our algorithm terminates in finite time with probability $1$. This guarantee is somewhat challenging to prove; see Section~\ref{sec:termination}.

\begin{figure*}
\small
\begin{minipage}{\textwidth}
\[\dfrac{\mu_Z:(E,\varepsilon,\delta)\in\Gamma}{\Gamma\vdash\mu_Z:(E,\varepsilon,\delta)}~(\text{random variable})\hspace{0.2in}
\dfrac{c\in\mathbb{R}}{\Gamma\vdash(c,0,0)}~(\text{constant})\hspace{0.2in}
\dfrac{\Gamma\vdash X:(E,\varepsilon,\delta),~\Gamma\vdash X':(E',\varepsilon',\delta')}{\Gamma\vdash X+X':(E+E',\varepsilon+\varepsilon',\delta+\delta')}~(\text{sum})\] \\
\[\dfrac{\Gamma\vdash X:(E,\varepsilon,\delta)}{\Gamma\vdash -X:(-E,\varepsilon,\delta)}~(\text{negative})\hspace{0.2in}
\dfrac{\Gamma\vdash X:(E,\varepsilon,\delta),~|E|>\varepsilon}{\Gamma\vdash X^{-1}:(E^{-1},\frac{\varepsilon}{|E|\cdot(|E|-\varepsilon)},\delta)}~(\text{inverse})\] \\
\[\dfrac{\Gamma\vdash X:(E,\varepsilon,\delta),~\Gamma\vdash X':(E',\varepsilon',\delta')}{\Gamma\vdash X\cdot X':(E\cdot E',|E|\cdot\varepsilon'+|E'|\cdot\varepsilon+\varepsilon\cdot\varepsilon',\delta+\delta')}~(\text{product})\] \\
\[\dfrac{\Gamma\vdash X:(E,\varepsilon,\delta),~E-\varepsilon\ge0}{\Gamma\vdash X\ge0:(\true,\delta)}~(\text{inequality true})\hspace{0.2in}
\dfrac{\Gamma\vdash X:(E,\varepsilon,\delta),~E+\varepsilon<0}{\Gamma\vdash X\ge0:(\false,\delta)}~(\text{inequality false})\] \\
\[\dfrac{\Gamma\vdash Y:(I,\gamma),~\Gamma\vdash Y':(I',\gamma')}{\Gamma\vdash Y\wedge Y':(I\wedge I',\gamma+\gamma')}~(\text{and})\hspace{0.15in}
\dfrac{\Gamma\vdash Y:(I,\gamma),~\Gamma\vdash Y':(I',\gamma')}{\Gamma\vdash Y\vee Y':(I\vee I',\gamma+\gamma')}~(\text{or})\hspace{0.15in}
\dfrac{\Gamma\vdash Y:(I,\gamma)}{\Gamma\vdash\neg Y:(\neg I,\gamma)}~(\text{not})\]
\caption{Inference rules used to derive lemmas $X:(E,\varepsilon,\delta)$ and $Y:(I,\gamma)$ for specifications $X\in\mathcal{L}(T)$ and $Y\in\mathcal{L}(S)$.}
\label{fig:inference}
\end{minipage}
\end{figure*}

\subsection{Concentration for Specifications}
\label{sec:inference}

Now, we describe how our algorithm derives estimates $E$ for $\llbracket X\rrbracket$ (where $X\in\mathcal{L}(T)$) and estimates $I$ for $\llbracket Y\rrbracket$ (where $Y\in\mathcal{L}(S)$), as well as high-probability bounds on these estimates. We use the notation $X:(E,\varepsilon,\delta)$ to denote that $E\in\mathbb{R}$ is an estimate for $\llbracket X\rrbracket$ with corresponding high-probability bound
\begin{align}
\label{eqn:xlemma}
\text{Pr}[|E-\llbracket X\rrbracket|\le\varepsilon]\ge1-\delta,
\end{align}
where $\varepsilon,\delta\in\mathbb{R}_+$. We call Eq.~\ref{eqn:xlemma} a \emph{lemma}. Similarly, we use the notation $Y:(I,\gamma)$ to denote that $I\in\{\true,\false\}$ is an estimate of $\llbracket Y\rrbracket$ with corresponding high-probability bound
\begin{align}
\label{eqn:ylemma}
\text{Pr}[I=\llbracket Y\rrbracket]\ge1-\gamma,
\end{align}
where $\gamma\in\mathbb{R}_+$. Then, let $\Gamma=\{\mu_Z:(\hat{\mu}_Z,\varepsilon,\delta)\}$ be an environment of lemmas for the subexpressions $\mu_Z$. In Figure~\ref{fig:inference}, we show the inference rules that our algorithm uses to derive lemmas for expressions $X\in\mathcal{L}(T)$ and $Y\in\mathcal{L}(S)$ given $\Gamma$. The rules for expectations $\mu_Z$ and constants $c$ are straightforward. Next, consider the rule for sums---its premise is
\begin{align*}
\text{Pr}[|E-\llbracket X\rrbracket|\le\varepsilon]&\ge1-\delta \\
\text{Pr}[|E'-\llbracket X'\rrbracket|\le\varepsilon']&\ge1-\delta'.
\end{align*}
By a union bound, the events $|E-\llbracket X\rrbracket|\le\varepsilon$ and $|E'-\llbracket X'\rrbracket|\le\varepsilon'$ hold with probability at least $1-(\delta+\delta')$, so
\begin{align*}
|(E+E')-(\llbracket X\rrbracket+\llbracket X'\rrbracket)|
&\le|E-\llbracket X\rrbracket|+|E'-\llbracket X'\rrbracket| \\
&\le\varepsilon+\varepsilon'.
\end{align*}
Thus, we have lemma $X+X':(E+E',\varepsilon+\varepsilon',\delta+\delta')$, which is exactly the conclusion of the rule for sums. The rules for products, inverses, and if-then-else statements hold using similar arguments; the only subtlety is that for inverses, a constraint $|E|>\varepsilon$ in the premise of the rule is needed to ensure that that $\llbracket X\rrbracket\neq0$ with probability at least $1-\delta$. The rules for conjunctions, disjunctions, and negations also follow using similar arguments. There are two rules for inequalities $X\ge0$---one for the case where the inequality evaluates to \true, and one for the case where it evaluates to \false. Note that at most one rule may apply (but it may be the case that neither rule applies). We describe the rule for the former case; the rule for the latter case is similar.

Note that the inequality evaluates to true as long as $\llbracket X\rrbracket\ge0$. Thus, suppose that $E$ is an estimate of $X$ satisfying the premise of the rule, i.e.,
\begin{align*}
\text{Pr}[|E-\llbracket X\rrbracket|\le\varepsilon]&\ge1-\delta \\
E-\varepsilon&\ge0.
\end{align*}
Rearranging the inequality $E-\llbracket X\rrbracket\le\varepsilon$ gives
\begin{align*}
\llbracket X\rrbracket\ge E-\varepsilon\ge0.
\end{align*}
Thus, $\llbracket X\ge0\rrbracket=\true$ with probability at least $1-\delta$ (since the original inequality holds with probability at least $1-\delta$). In other words, we can conclude that $X\ge0:(\true,\delta)$, which is exactly the conclusion of the rule for the inequality evaluating to \true. In summary, we have:
\begin{theorem}
\label{thm:sound}
The inference rules in Figure~\ref{fig:inference} are sound.
\end{theorem}
We give a proof in Appendix~\ref{sec:soundproof}. As an example, we describe how to apply the inference rules to infer whether the demographic parity specification $Y_{\text{parity}}$ holds. Recall that this specification is a function of the Bernoulli random variables $\rmaj$ and $\rmin$. Suppose that
\begin{align*}
  \mu_{\rmaj}&:(E_{\text{maj}},\varepsilon_{\text{maj}},\delta_{\text{maj}}) \\
  \mu_{\rmin}&:(E_{\text{min}},\varepsilon_{\text{min}},\delta_{\text{min}}),
\end{align*}
and that $|E_{\text{maj}}|>\varepsilon_{\text{maj}}$. Let
\begin{align*}
  E_{\text{parity}}&=E_{\text{min}}\cdot E_{\text{maj}}^{-1}-(1-c) \\
  \varepsilon_{\text{parity}}&=|E_{\text{maj}}|^{-1}\cdot\varepsilon_{\text{min}}+\frac{\varepsilon_{\text{maj}}\cdot(|E_{\text{min}}|+\varepsilon_{\text{min}})}{|E_{\text{maj}}|(|E_{\text{maj}}|-\varepsilon_{\text{maj}})}
\end{align*}
Now, if $E_{\text{parity}}-\varepsilon_{\text{parity}}\ge0$, then $Y_{\text{parity}}:(\true,\delta_{\text{maj}}+\delta_{\text{min}})$, and if $E_{\text{parity}}+\varepsilon_{\text{parity}}<0$, then $Y_{\text{parity}}:(\false,\delta_{\text{maj}}+\delta_{\text{min}})$.

\begin{figure}
\small
\begin{minipage}{0.2\textwidth}
\begin{align*}
\llbracket\mu_Z\rrbracket_{\delta}&=1 \\
\llbracket c\rrbracket_{\delta}&=0 \\
\llbracket X+X'\rrbracket_{\delta}&=\llbracket X\rrbracket_{\delta}+\llbracket X'\rrbracket_{\delta} \\
\llbracket-X\rrbracket_{\delta}&=\llbracket X\rrbracket_{\delta} \\
\llbracket X\cdot X'\rrbracket_{\delta}&=\llbracket X\rrbracket_{\delta}+\llbracket X'\rrbracket_{\delta}
\end{align*}
\end{minipage}
\hspace{1.0in}
\begin{minipage}{0.2\textwidth}
\begin{align*}
\llbracket X^{-1}\rrbracket_{\delta}&=\llbracket X\rrbracket_{\delta} \\
\llbracket X\ge0\rrbracket_{\delta}&=\llbracket X\rrbracket_{\delta} \\
\llbracket X\wedge X'\rrbracket_{\delta}&=\llbracket X\rrbracket_{\delta}+\llbracket X'\rrbracket_{\delta} \\
\llbracket X\vee X'\rrbracket_{\delta}&=\llbracket X\rrbracket_{\delta}+\llbracket X'\rrbracket_{\delta} \\
\llbracket\neg X\rrbracket_{\delta}&=\llbracket X\rrbracket_{\delta}
\end{align*}
\end{minipage}
\caption{Inference rules used to compute $\delta_Z$, in particular, $\delta_Z=\Delta/\llbracket Y\rrbracket_{\delta}$, where $Y\in\mathcal{L}(S)$ is the specification to be verified and $\delta\in\mathbb{R}_+$ is the desired confidence.}
\label{fig:choosingdelta}
\end{figure}

\subsection{Choosing $\delta_Z$}
\label{sec:choosingdelta}

To ensure that Algorithm~\ref{alg:verify} terminates, we have to ensure that for any given problem instance, we eventually either prove or disprove the given specification $Y$.\footnote{We require a technical condition on the problem instance; see Section~\ref{sec:theory}.} More precisely, as $n\to\infty$ (where $n$ is the number of samples taken so far), we must derive $\Gamma\vdash Y:(I,\gamma)$ for some $\gamma\le\Delta$ (where $\Delta$ is the given confidence level) and $I\in\{\true,\false\}$, with probability $1$. In particular, the value $\gamma$ depends on the environment $\Gamma=\{\mu_Z:(s/n,\varepsilon_Z,\delta_Z)\}$. In $\Gamma$, our algorithm can choose the value $\delta_Z\in\mathbb{R}_+$ (which determines $\varepsilon_Z=\varepsilon(\delta_Z,n)$ via Eq.~\ref{eqn:epsilon}). Thus, to ensure termination, we have to choose $\delta_Z$ so that we eventually derive $Y:(I,\gamma)$ such that $\gamma\le\Delta$.

In fact, $\gamma$ is a simple function of $\delta_Z$---each inference rule in Figure~\ref{fig:inference} adds the values of $\delta$ (or $\gamma$) for each subexpression of the current expression, so $\gamma$ equals the sum of the values of $\delta$ for each leaf in the syntax tree of $Y$. Since we have assumed there is a single Bernoulli random variable $Z$, each leaf in the syntax tree has either $\delta=\delta_Z$ (for leaves labeled $\mu_Z$) or $\delta=0$ (for leaves labeled $c\in\mathbb{R}$). Thus, $\gamma$ has the form $\gamma=m\cdot\delta_Z$ for some $m\in\mathbb{N}$. The rules in Figure~\ref{fig:choosingdelta} compute this value $m=\llbracket Y\rrbracket_{\delta}$---the base cases are $\llbracket\mu_Z\rrbracket_{\delta}=1$ and $\llbracket c\rrbracket_{\delta}=0$, and the remaining rules add together the values of $m$ for each subexpression of the current expression.

As a consequence, for any $\Delta\in\mathbb{R}_+$, we can derive $Y:(I,\gamma)$ with $\gamma\le\Delta$ from $\Gamma$ by choosing $\delta_Z=\Delta/m$.
\begin{theorem}
\label{thm:inferenceterminate}
Let $(P_{\Z},Y)$ be a well-defined problem instance, and let $\Delta\in\mathbb{R}_+$ be arbitrary. Let $\delta_Z=\Delta/\llbracket Y\rrbracket_{\delta}$, and let
\begin{align*}
  \Gamma^{(n)}=\{\mu_Z:(E^{(n)},\varepsilon(\delta_Z,n),\delta_Z)\}
\end{align*}
be the lemma established on the $n$th iteration of Algorithm~\ref{alg:verify} (i.e., using $n$ random samples $Z\sim P_{\Z}$). Then, for any $\delta_0\in\mathbb{R}_+$, there exists $n_0\in\mathbb{N}$ such that for all $n\ge n_0$, we have
\begin{align*}
  \Gamma^{(n)}\vdash Y:(I,\gamma)
\end{align*}
where $\gamma\le\Delta$ with probability at least $1-\delta_0$.
\end{theorem}
We give a proof in Appendix~\ref{sec:inferenceterminateproof}. Note that the probability is taken over the $n$ random samples $Z\sim P_{\Z}$ used to construct $E^{(n)}$. Also, note that the success of the inference is in a high-probability, asymptotic sense---this approach is necessary since adversarial sequences of random samples $Z\sim P_{\Z}$ may cause nontermination, but the probability of such adversarial samples becomes arbitrarily small as $n\to\infty$. Finally, we have focused on the case where there is a single Bernoulli random variable $\mu_Z$. In the general case, we use the same $\delta_Z=\Delta/\llbracket Y\rrbracket_{\delta}$ for each Bernoulli random variable $Z$; Theorem~\ref{thm:inferenceterminate} follows with exactly the same reasoning.

Continuing our example, we describe how $\delta_{\rmaj}$ and $\delta_{\rmin}$ are computed for $Y_{\text{parity}}$. In particular, the inference rules in Figure~\ref{fig:choosingdelta} give $\llbracket Y_{\text{parity}}\rrbracket_{\delta}=2$, so it suffices to choose
\begin{align*}
\delta_{\rmaj}=\delta_{\rmin}=\frac{\Delta}{2}.
\end{align*}
Recall from Section~\ref{sec:inference} that we actually have $\gamma=\delta_{\rmaj}+\delta_{\rmin}$, so this choice indeed suffices to ensure that $\gamma\le\Delta$.

\section{Theoretical Guarantees}
\label{sec:theory}

We prove that Algorithm~\ref{alg:verify} terminates with probability $1$ as long as the given problem instance satisfies a technical condition. Futhermore, we prove that Algorithm~\ref{alg:verify} is sound and precise in a probabilistic sense.

\subsection{Termination}
\label{sec:termination}

Algorithm~\ref{alg:verify} terminates as long as it the given problem instance satisfies the following condition:
\begin{definition}
Given a problem instance consisting of an expression $W\in\mathcal{L}(T)\cup\mathcal{L}(S)$ together with a distribution $P_{\Z}$ for each $\mu_Z$ occuring in $W$, we say the problem instance is {\bf\em well-defined} if its subexpressions are well-defined. If $W\equiv (X\ge0)$ or $W\equiv X^{-1}$, we furthermore require that $\llbracket X\rrbracket\neq0$.
\end{definition}
If $Y$ contains a subexpression $X^{-1}$ such that $\llbracket X\rrbracket=0$, then $\llbracket X^{-1}\rrbracket$ is infinite. As a consequence, Algorithm~\ref{alg:verify} fails to terminate since it cannot estimate of $\llbracket X^{-1}\rrbracket$ to any finite confidence level. Next, the constraint on subexpressions of the form $X\ge0$ is due to the nature of our problem formulation. In particular, consider an expression $X\ge0$, where $\llbracket X\rrbracket=0$. In our setting, we cannot compute $\llbracket\mu_Z\rrbracket$ exactly since we are treating the Bernoulli random variables $Z,...$ as blackboxes. Therefore, we also cannot compute $\llbracket X\rrbracket$ exactly (assuming it contains subexpressions of the form $\mu_Z$). Thus, we can never determine with certainty whether $\llbracket X\rrbracket\ge0$.
\begin{theorem}
\label{lem:terminate}
Given a well-defined problem instance, Algorithm~\ref{alg:verify} terminates with probability $1$, i.e.,
\begin{align*}
  \lim_{n\to\infty}\text{Pr}[\text{Algorithm }~\ref{alg:verify}\text{ terminates}]=1,
\end{align*}
where $n$ is the number of samples taken so far.
\end{theorem}
The proof of this thoerem is somewhat subtle. In particular, our algorithm only terminates if we can prove that $\hat{\mu}_Z^{(n)}\to\mu_Z$ as $n\to\infty$, where $\hat{\mu}_Z^{(n)}$ is the estimate of $\mu_Z$ established on the $n$th iteration. However, we cannot use our adaptive concentration inequality in Theorem~\ref{thm:bernoulliconcentration} to prove this guarantee, since our adaptive concentration inequality \emph{assumes} that our algorithm terminates with probability $1$. Thus, we have to directly prove that our estimates converge, and then use this fact to prove that our algorithm terminates. We give a full proof in Appendix~\ref{sec:terminateproof}.

The restriction to well-defined properties is not major---for typical problem instances, having $\llbracket X\rrbracket=0$ hold exactly is very unlikely. Furthermore, this restriction to well-defined problem instances is implicitly assumed by current state-of-the-art systems, including \fairsquare~\cite{albarghouthi2017fairsquare}. In particular, it is a necessary restriction for any system that does not exactly evaluate the expectations $\mu_Z$. For example, \fairsquare relies on a technique similar to numerical integration, and can only obtain estimates $\mu_Z\in[E-\varepsilon,E+\varepsilon]$; therefore, it will fail to terminate given an ill-defined problem instance.

\subsection{Probabilistic Soundness and Precision}

Let $Y\in\mathcal{L}(S)$ be a specification, and consider a verification algorithm tasked with computing $\llbracket Y\rrbracket$. Typically, the algorithm is sound if it only returns \true when $\llbracket Y\rrbracket=\true$, and it is precise if it only returns \false when $\llbracket Y\rrbracket=\false$. However, because our algorithm uses random samples to evaluate $\llbracket Y\rrbracket$, it cannot guarantee soundness or precision---e.g., adversarial sequences of samples can cause the algorithm to fail. Instead, we need probabilistic notions of soundness and precision.
\begin{definition}
\label{def:sound}
Let $\Delta\in\mathbb{R}_+$. We say a verification algorithm is {\bf\em$\Delta$-sound} if it returns $\true$ only if
\begin{align*}
\text{Pr}[\llbracket Y\rrbracket=\true]\ge1-\Delta,
\end{align*}
where the probability is taken over the random samples drawn by the algorithm. Furthermore, if the algorithm takes $\Delta$ as a parameter, and is $\Delta$-sound for any given $\Delta\in\mathbb{R}_+$, then we say that the algorithm is {\bf\em probabilistically sound}.
\end{definition}
\begin{definition}
\label{def:precision}
Let $\Delta\in\mathbb{R}_+$. We say a verification algorithm is {\bf\em$\Delta$-precise} if it returns $\false$ only if
\begin{align*}
\text{Pr}[\llbracket Y\rrbracket=\false]\ge1-\Delta
\end{align*}
where the probability is taken over the random samples drawn by the algorithm. Furthermore, if the algorithm takes $\Delta$ as a parameter, and is $\Delta$-precise for any given $\Delta\in\mathbb{R}_+$, then we say that the algorithm is {\bf\em probabilistically precise}.
\end{definition}
\begin{theorem}
\label{thm:terminate}
Algorithm~\ref{alg:verify} is probabilistically sound and probabilistically precise.
\end{theorem}
We give a proof in Appendix~\ref{sec:mainproof}. For ill-defined problem instances, Algorithm~\ref{alg:verify} may fail to terminate, but nontermination is allowed by probabilistic soundness and precision.

\begin{figure*}
\begin{tabular}{cc}
\includegraphics[width=0.45\textwidth]{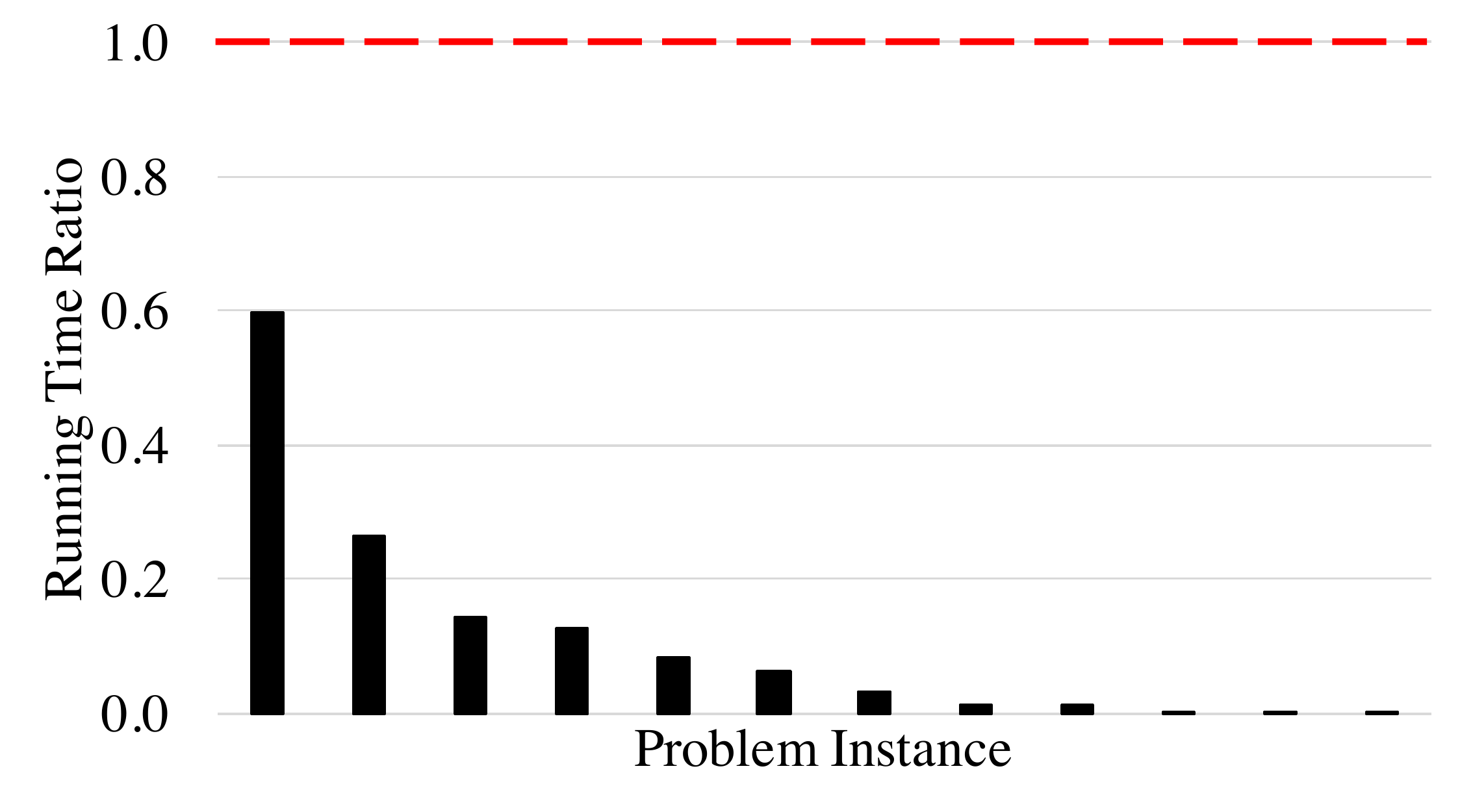} &
\includegraphics[width=0.45\textwidth]{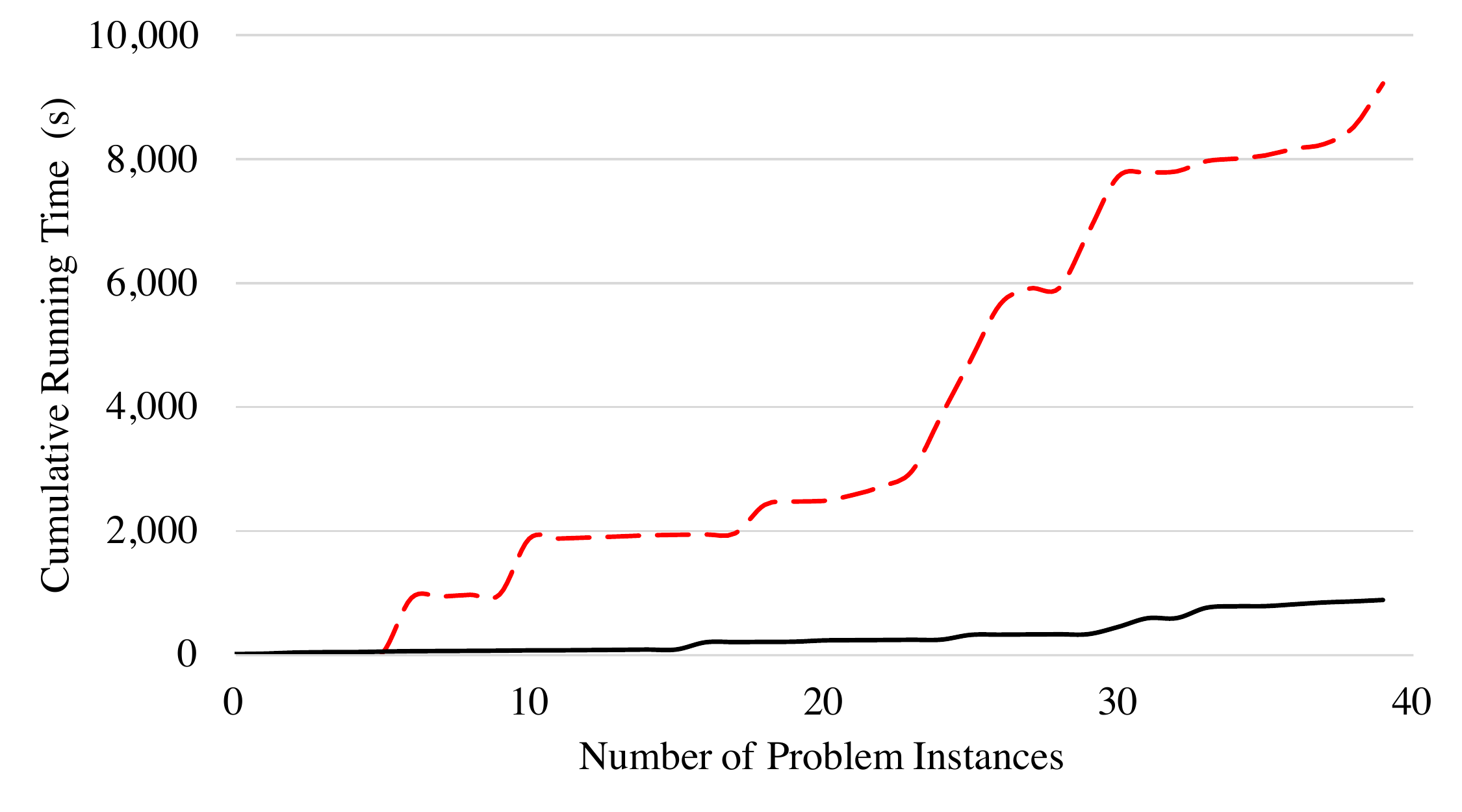} \\
(a) & (b)
\end{tabular}
\caption{(a) Results for the largest problem instances from the \fairsquare benchmark. The $y$-axis is the ratio of the \toolname running time to the \fairsquare running time (so lower is better). The problem instances are along the $x$-axis; we have sorted them from highest to lowest. The red, dashed line at $y=1$ denotes the \fairsquare running time; for all instances below this line, \toolname outperforms \fairsquare. (b) The cumulative running time of \toolname (black, solid) and \fairsquare (red, dashed). In particular, we sorted all 39 problem instances from smallest to largest (in terms of lines of code), and plot the cumulative running time from running the first $i$ benchmarks. The $x$-axis is $i$, and the $y$-axis is the running time.}
\label{fig:fairsquareplots}
\end{figure*}

\section{Evaluation}
\label{sec:evaluation}

We have implemented our algorithm a tool called \toolname, which we evaluate on two benchmarks. First, we compare to \fairsquare on their benchmark, where the goal is to verify whether demographic parity holds~\cite{albarghouthi2017fairsquare}. In particular, we show that \toolname scales substantially better than \fairsquare on every large problem instance in their benchmark (with $\Delta=10^{-10}$).

However, the \fairsquare benchmark fails to truly demonstrate the scalability of \toolname. In particular, it exclusively contains tiny classifiers---e.g., the largest neural network in their benchmark has a single hidden layer with just two hidden units. This tiny example already causes \fairsquare to time out. Indeed, the scalability of \fairsquare depends on the complexity the internal structure of the classifier and population model, whereas the scalability of \toolname only depends on the time it takes to execute these models.

Thus, for our second benchmark, we use a state-of-the-art deep recurrent neural network (RNN) designed to classify sketches~\cite{drawclassify}, together with a state-of-the-art deep generative model for randomly sampling sketches similar to those produced by humans~\cite{ha2017neural}. Together, these two deep neural networks have more than 16 million parameters, which is 5 orders of magnitude larger than the largest neural network in the \fairsquare benchmark. We show that \toolname scales to this benchmark, and furthermore study how its scalability depends on various hyperparameters. In fact, \fairsquare cannot even be applied to this benchmark, since \fairsquare can only be applied to straight line programs but the RNN computation involves a possibly unbounded loop operation.

\subsection{FairSquare Benchmark}

We begin by comparing our tool, \toolname, to \fairsquare, a state-of-the-art fairness verification tool. The results on this benchmark were run on a machine with a 2.2GHz Intel Xeon CPU with 20 cores and 128 GB of memory.

\paragraph{{\bf\em Benchmark.}}

The \fairsquare benchmark contains 39 problem instances. Each problem instance consists of a classifier $f:\V\to\{0,1\}$, where $\V=\mathbb{R}^d$ with $d\in[1,6]$, together with a population model encoding a distribution $P_{\V}$ over $\V$. The classifiers include decision trees with up to 44 nodes, SVMs, and neural networks with up to 2 hidden units. The population models include one where the features are assumed to be independent and two Bayes net models. The goal is to check whether demographic parity holds, taking $c=0.15$ in Definition~\ref{def:demographicparity}. We run \toolname using $\Delta=10^{-10}$ (i.e., the probability of an incorrect response is at most $10^{-10}$).

In theory, \fairsquare provides stronger guarantees than \toolname, since \fairsquare never responds incorrectly. Intuitively, the guarantees provided by \fairsquare are analogous to using \toolname with $\Delta=0$. However, as we discuss below, because we have taken the parameters to be so small, they have essentially no effect on the outputs of \toolname. Also, the population models in the \fairsquare benchmark often involve conditional probabilities. There are many ways to sample such a probability distribution. We use the simplest technique, i.e., rejection sampling; we discuss the performance implications below. Finally, the problem instances in the \fairsquare benchmark are implemented as Python programs. While we report results using the original Python implementations, below we discuss how compiling the benchmarks can substantially speed up execution.

\input{results}

\paragraph{{\bf\em Results.}}

For both tools, we set a timeout of 900 seconds. We give a detailed results in Table~\ref{tab:fairsquare}. For each problem instance, we show the running times of \toolname and \fairsquare, as well as the ratio
\begin{align*}
  \frac{\text{running time of \toolname}}{\text{running time of \fairsquare}},
\end{align*}
where we conservatively assume that \fairsquare runs in 900 seconds for problem instances in which it times out. We also show the number of lines of code and some statistics about the rejection sampling approach we use to sample the population models.

\toolname outperforms \fairsquare on 30 of the 39 problem instances. More importantly, \toolname scales much better to large problem instances---whereas \fairsquare times out on 4 problem instances, \toolname terminates on all 39 in within 200 seconds. In particular, while \fairsquare relies on numerical integration that may scale exponentially in the problem size, \toolname relies on sampling, which linearly in the time required to execute the population model and classifier.

In Figure~\ref{fig:fairsquareplots} (a), we show results for 12 of the largest problem instances. In particular, we include the largest two each of decision tree, SVM, and neural network classifiers, using each of the two Bayes net population models. As can be seen, \toolname runs faster than \fairsquare on all of the problem instances, and more than twice as fast in all but one.

Similarly, in Figure~\ref{fig:fairsquareplots} (b), we plot the cumulative running time of each tool across all 39 problem instances. For this plot, we sort the problem instances from smallest to largest based on number of lines of code. Then, the plot shows the cumulative running time of the first $i$ problem instances, as a function of $i$. As before, we conservatively assume that \fairsquare terminates in 900 seconds when it times out. As can be seen, \toolname scales significantly better than \fairsquare---\toolname becomes faster than \fairsquare after the first 9 problem instances, and substantially widens that lead as the problem instances become larger.

\paragraph{{\bf\em Compiled problem instances.}}

The running time of \toolname depends linearly on the time taken by a single execution of the population model and classifier. Because the benchmarks are implemented in Python, the running time can be made substantially faster if they are compiled to native code. To demonstrate this speed up, we manually implement two of the problem instances in C++:
\begin{itemize}
\item The decision tree with 14 nodes with the independent population model; in this problem instance, \toolname is slowest relative to \fairsquare (29.3$\times$ slower). \toolname runs the compiled version of this model in just 0.40 seconds, which is a 301$\times$ speed up, and more than 10$\times$ faster than \fairsquare.
\item The decision tree with 14 nodes with Bayes net 2 as the popluation model; in this problem instance, \toolname is slowest overall (190.1 seconds). \toolname runs the compiled version of this model in just 0.58 seconds, which is a 327$\times$ speed up.
\end{itemize}
Note that compiling problem instances would not affect \fairsquare, since it translates them to SMT formula.

\paragraph{{\bf\em Comparison of guarantees.}}

We ran the \toolname ten times on the benchmark; the responses were correct on all iterations. Indeed, because we have set $\Delta=10^{-10}$, it is extremely unlikely that the response of \toolname is incorrect.

\paragraph{{\bf\em Rejection sampling.}}

When the population model contains conditional probabilities, \toolname uses rejection sampling to sample the model. The acceptance rate is always between 20-21\%. This consistency is likely due to the fact that the models in the \fairsquare benchmark are always modeling the same population. Thus, rejection sampling is an effective strategy for the \fairsquare benchmark. Furthermore, we discuss possible alternatives to rejection sampling in Section~\ref{sec:discussion}.

\paragraph{{\bf\em Path-specific causal fairness.}}

We check whether path-specific causal fairness $Y_{\text{causal}}$ holds for three \fairsquare problem instances---the largest classifier of each kind using the Bayes net 2 population model. We use the number of years of education as the mediator covariate. We use $\Delta=10^{-10}$. \toolname concludes that all of the problem instances are fair. The running time for the decision tree $\text{DT}_{44}$ with 44 nodes is 2.47 seconds, for the SVM $\text{SVM}_{6}$ with $d=6$ features terminates is 8.89 seconds, and for the neural network $\text{NN}_{3,2}$ with $d=3$ features and 2 neurons is 0.35 seconds.

\begin{figure*}
\begin{tabular}{cc}
\includegraphics[width=0.45\textwidth]{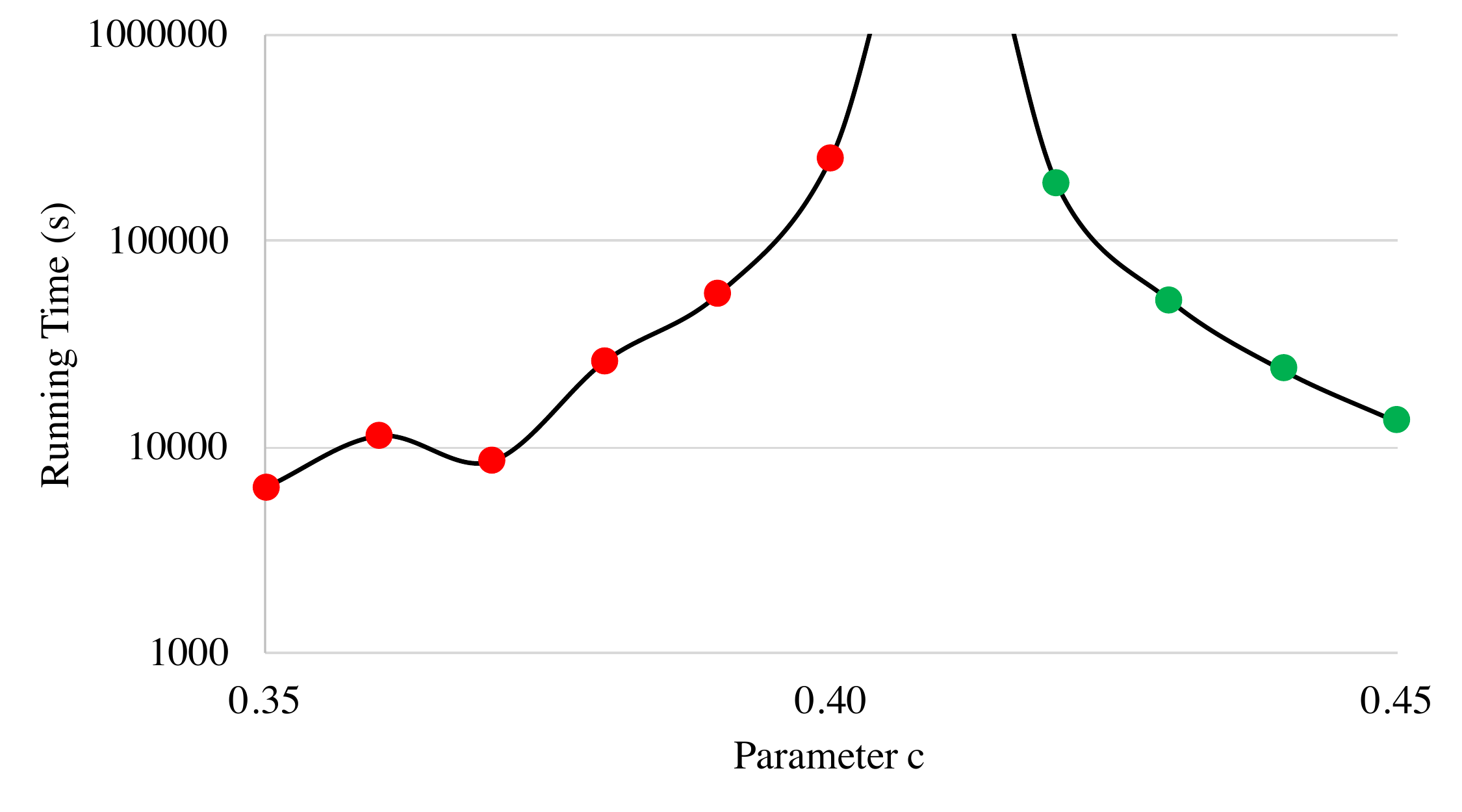} &
\includegraphics[width=0.45\textwidth]{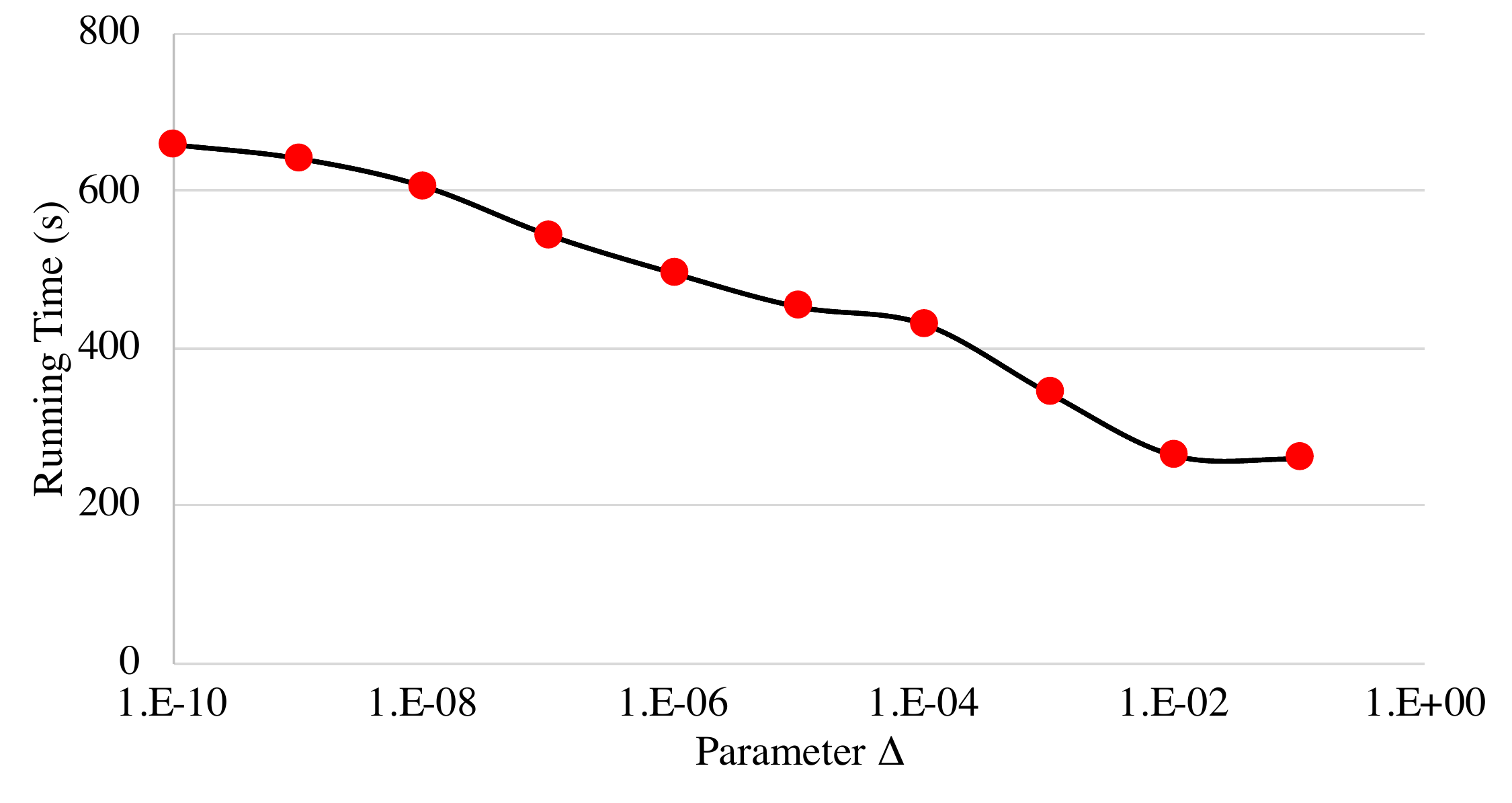} \\
(a) & (b)
\end{tabular}
\caption{We plot the running time of \toolname on the Quick Draw benchmark as a function of (a) the parameter $c$, and (b) the parameter $\Delta$. The running times in (b) are averaged over 10 runs; the running times in (a) are reported for a single run since they were too expensive to run multiple times. In each figure, a green marker denotes a response of ``fair'' and a red marker denotes a response of ``unfair''. In (a), the curve diverges because \toolname times out when $c=0.41$.}
\label{fig:quickdrawplots}
\end{figure*}

\subsection{Quick Draw Benchmark}

Our next benchmark consists of a deep recurrent neural network (RNN) classifier and a deep sequence-to-sequence variational autoencoder (VAE) population model~\cite{ha2017neural}. Recall that \toolname scales linearly with the running time of the classifier and population model; therefore, \toolname should scale to these very complex models as long as executing the model can be executed in a reasonable amount of time. In this benchmark, we use \toolname to verify the equal opportunity property described in Definition~\ref{def:equalopportunity}. Finally, we study how the running time of \toolname on this benchmark depends on various problem parameters. The results on this benchmark were run on a machine with an Intel Core i7-6700K 4GHz quad core CPU, an Nvidia GeForce GTX 745 with 4GB of GPU memory (used to run the deep neural networks), and 16GB of memory. Note that we cannot run \fairsquare on this benchmark, since it can only handle straight-line models but recurrent neural networks involve a loop operation.

\paragraph{{\bf\em Benchmark.}}

The classifier in our benchmark is an RNN $f:\X\to\{0,1\}$, where $x\in\X$ is representation of a $256\times256$ image that is a black and white sketch drawn by a human. This image is represented as a sequence of strokes $(x,y,p)$, where $(x,y)$ is the displacementand $p$ is a command (pen down, pen up, or finish drawing). Each input is a drawing of one of 345 different categories of objects, including dogs, cats, firetrucks, gardens, etc. To obtain a binary classifier, we train a classifier to predict a binary label $y\in\{0,1\}$ indicating whether the input image is a drawing of a dog. The neural network was trained on a dataset $X\subseteq\X$ containing 70K training examples, 2.5K cross-validation examples, and 2.5K test examples; overall, 0.3\% of the training images are dogs. Its accuracy is 18.8\%, which is 626$\times$ better than random. Our population model is the decoder portion of a VAE~\cite{ha2017neural}, which generates a random sequence in the form described above.

We have the country of origin for each image; we consider images from the United States to be the majority subpopulation (43.9\% of training examples), and images from other countries to be the minority subpopulation. We train two population models: (i) a decoder $d_{\text{maj}}$ trained to generate sketches of dogs from the United States, and (ii) $d_{\text{min}}$ to generate sketches of dogs from other countries.

We aim to verify the equal opportunity property. Recall that this property says that the classifier $f$ should not make mistakes (in particular, false negatives) much more frequently for members of the minority class than for members of the majority class. For example, a classifier that always responds randomly is fine, but the classifier cannot respond accurately for majority members and randomly for minority members. In the context of the Quick Draw benchmark, this fairness property says that the classifier should not perform worse for people from outside of the United States. This guarantee is important for related tasks---e.g., classifiers for detecting skin cancer from photographs~\cite{esteva2017dermatologist} and login systems based on face recognition~\cite{simon2009hp}; for example, certain face recognition systems have been shown to have more difficulty detecting minority users than detecting majority users. As before, we use parameter $c=0.15$ in the fairness specification.

\paragraph{{\bf\em Batched samples.}}

Typical deep learning frameworks are much more efficient when they operate on batches of data. Thus, we batch the samples taken by \toolname---on each iteration, it samples 1000 images $X_{\text{maj}}\sim d_{\text{maj}}$ and 1000 images $X_{\text{min}}\sim d_{\text{min}}$ as a batch, and computes $f(X_{\text{maj}})$ and $f(X_{\text{min}})$ as a batch as well. As a consequence, Algorithm~\ref{alg:verify} may sample up to 999 more images than needed, but we find that execution time improves significantly---sampling a single image takes about 0.5 seconds, whereas sampling 1000 images takes about 30 seconds, for a speed up of about $17\times$.

\paragraph{{\bf\em Results.}}

We ran \toolname on our benchmark; using $\Delta=10^{-5}$, \toolname terminates in 301 seconds and uses 14,000 samples, and using $\Delta=10^{-10}$, \toolname terminated in 606 seconds and uses 28,000 samples.

\paragraph{{\bf\em Varying $c$.}}

Besides the running time of the classifier $f$ and population models $d_{\text{maj}}$ and $d_{\text{min}}$, the most important factor affecting the running time of \toolname is the value of the parameter $c$. In particular, in the specification
\begin{align*}
Y_{\text{equal}}\equiv\left(\frac{\mu_{\rmin}}{\mu_{\rmaj}}\ge1-c\right),
\end{align*}
as the left-hand side and right-hand side of the inequality become closer together, then we need increasingly accurate estimates of $\mu_{\rmin}$ and $\mu_{\rmaj}$ to check whether the specification holds. Thus, \toolname needs to take a larger number of samples to confidently determine whether $Y_{\text{equal}}$ holds.

We ran \toolname on with values of $c$ near
\begin{align*}
c_0=1-\frac{\mu_{\rmin}}{\mu_{\rmaj}}\approx0.41,
\end{align*}
in particular, $c\in\{0.35,0.36,...,0.45\}$ (with $\Delta=10^{-5}$). In Figure~\ref{fig:quickdrawplots} (a), we plot the running time of \toolname on Quick Draw as a function of $c$. \toolname terminated for all choices of $c$ except $c=0.41$, which timed out after 96 hours. For the remaining choices of $c$, the longest running time was $c=0.40$, which terminated after 84 hours. We also show whether \toolname concludes that the specification is true (green marker) or false (red marker)---\toolname concludes that Quick Draw is fair if $c>0.41$ and unfair if $c\le0.40$.

In practice, $c$ is unlikely to be very close to $c_0$. Furthermore, approaches based on numerical integration would suffer from a similar divergence near $c=c_0$, since their estimate of $Y_{\text{equal}}$ is subject to numerical errors that must be reduced by increasing precision, which increases running time.

\paragraph{{\bf\em Varying $\Delta$.}}

We study the running time of \toolname on Quick Draw as a function of $\Delta$, which controls the probability that \toolname may respond incorrectly. In particular, we ran \toolname on Quick Draw with values $\Delta\in\{10^{-10},10^{-9},...,10^{-1}\}$ (with $c=0.15$). In Figure~\ref{fig:quickdrawplots} (b), we plot the running time of \toolname as a function of $\Delta$. As expected, the running time increases as $\Delta$ becomes smaller.
%Note that the trend is not strict since the running time of \toolname is somewhat stochastic; the running time would be strictly monotonically decreasing if we averaged the running time over many iterations.
Even using $\Delta=10^{-10}$, the running time is only about 10 minutes. In particular, \toolname scales very well as a function of $\Delta$---the running time only increases linearly even as we decrease $\Delta$ exponentially.

\section{Discussion}
\label{sec:discussion}

In this section, we discuss various aspects of our algorithm.

\paragraph{{\bf\em Population models.}}

A key input to our algorithm is the population model encoding the distribution over population members. Intuitively, population models are analogous to preconditions. Population models are required for most fairness definitions, since these definitions are typically constraints on statistical properties of the classifier for different subpopulations. Without a population model, we cannot reason about the distribution of outputs. Population models can easily be obtained by fitting a density estimation model (e.g., a GAN, Bayesian network, VAE, etc.) to the data.

An advantage of our approach compared to previous work is that we only require blackbox access to the population model. Thus, if a population model is unavailable, our tool can actually be run online as real population members arrive over time. In this setting, it may be possible that an unfair model is deployed in production for some amount of time, but our tool will eventually detect the unfairness, upon which the model can be removed.

\paragraph{{\bf\em Additional fairness specifications.}}

While we have focused on a small number of fairness specifications, many others have been proposed in the literature. Indeed, the exact notion of fairness can be context-dependent; a major benefit of our approach is that it can be applied to a wide range of specifications. For example, we can straightforwardly support other kinds of fairness for supervised learning~\cite{kleinberg2016inherent,zafar2017fairness,galhotra2017fairness}. We can also straightforwardly extend our techniques to handle multiple minority subgroups; for example, the extension of demographic parity to a set $\mathcal{M}_{\text{min}}$ of minority subgroups is
\begin{align*}
Y_{\text{parity}}\equiv\bigwedge_{m\in\mathcal{M}_{\text{min}}}\left(\frac{\mu_{R_m}}{\mu_{\rmaj}}\ge1-c\right),
\end{align*}
where $R_m=f(V_m)$ and $V_m=P_{\V}\mid A=m$. Furthermore, we can support extensions of these properties to regression and multi-class classification; for example, for regression, an analog of demographic parity is
\begin{align*}
Y_{\text{reg}}\equiv|\mu_{\rmaj}-\mu_{\rmin}|\le c,
\end{align*}
where $f:\V\to[0,1]$ is a real-valued function (where $[0,1]$ is the unit interval), $R_{\text{maj}}=f(\vmaj)$ and $R_{\text{min}}=f(\vmin)$ are as before, and $c\in\mathbb{R}_+$ is a constant.
\footnote{The constraint that $f(V)\in[0,1]$ is needed for our concentration inequality, Theorem~\ref{thm:bernoulliconcentration}, to apply.}
In other words, the outcomes for members of the majority and minority subpopulations are similar on average. In the same way, we can support extensions to the reinforcement learning setting~\cite{wen2019fairness}. We can also support counterfactual fairness~\cite{kusner2017counterfactual} and causal fairness~\cite{kilbertus2017avoiding}, which are variants of path-specific causal fairness without a mediator variable.

Another approach to fairness is \emph{individual fairness}~\cite{dwork2012fairness}, which intuitively says that people with similar observed covariates should be treated similarly. Traditionally, this notion is defined over a finite set of individuals $x,y\in\X$, in which case it says:
\begin{align}
\label{eqn:individual}
\bigwedge_{V\in\V}\bigwedge_{V'\in\V}\|f(V)-f(V')\|_1\le\lambda\cdot\|V-V'\|_1
\end{align}
where $f(x)\in\mathbb{R}^k$ are the outcomes and $\lambda\in\mathbb{R}_+$ is a given constant. This finite notion is trivial to check by enumerating over $V,V'\in\V$. We can check an extension to the case of continous $V,V'\in\V$, except where we only want Eq.~\ref{eqn:individual} to hold with high probability:
\begin{align*}
Y_{\text{ind}}\equiv(\mu_R\ge1-c),
\end{align*}
where $c\in\mathbb{R}_+$ is a constant, and where $V\sim\V$, $V'\sim\V$, and
\begin{align*}
R=\mathbb{I}\left[\|f(V)-f(V')\|_1\le\|V-V'\|_1\right].
\end{align*}
In particular, note that $R$ is a Bernoulli random variable that indicates whether Eq.~\ref{eqn:individual} holds for a random pair $V,V'\sim\V$. Thus, the specification $Y_{\text{ind}}$ says that the probability that Eq.~\ref{eqn:individual} holds for random individuals $V$ and $V'$ is at least $1-c$.

\paragraph{{\bf\em Sampling algorithm.}}

Recall that \toolname uses rejection sampling, which we find works well for typical fairness definitions. In particular, most definitions only condition on being a member of the majority or minority subpopulation, or other fairly generic qualifications. These events are not rare, so there is no need to use more sophisticated sampling techniques. We briefly discuss how our approach compares to symbolic methods, as well as a possible approach to speeding up sampling by using importance sampling.

First, we note that existing approaches based on symbolic methods---in particular, \fairsquare; see Figure 6 in \cite{albarghouthi2017fairsquare}---would also have trouble scaling to specifications that condition on rare events. The reason is that \fairsquare requires that the user provides a piecewise constant distribution $\tilde{P}_{\V}$ that approximates the true distribution $P_{\V}$. Their approach performs integration by computing regions of the input space that have high probability according to this approximate distribution $\tilde{P}_{\V}$; once an input region is chosen, it computes the actual volume according to the true distribution $P_{\V}$. Thus, if the approximation $\tilde{P}_{\V}$ is poor, then the actual volume could be much smaller than the volume according to the approximation, so their approach would also scale poorly.

Furthermore, if our algorithm has access to a good approximation $\tilde{P}_{\V}$, then we may be able to use it to speed up sampling. In particular, suppose that we have access to a piecewise constant $\tilde{P}_{\V}$, where each piece is on a polytope $A_i$ (for $i\in[h]$) of the input space, and the probability on $A_i$ is a constant $p_i\in[0,1]$. We consider the problem of sampling from $P_{\V}\mid C$, where we assume (as in \fairsquare) that the constraints $C$ are affine. In this context, we can use $\tilde{P}_{\V}$ to improve the scalability of sampling by using \emph{importance sampling}. First, to sample $\tilde{P}_{\V}\mid C$, we can efficiently compute the volume of each of constrained polytope $v_i=\text{Vol}(A_i\cap A_C)$, where $A_C$ is the polytope corresponding to the constraints $C$~\cite{lawrence1991polytope}. Next, we can directly sample $V\sim\tilde{P}_{\V}$ as follows: (i) sample a random polytope according to their probabilities according to $\tilde{P}_{\V}\mid C$, i.e., $i\sim\text{Categorical}(p_1\cdot v_1,...,p_h\cdot v_h)$, and (ii) randomly sample a point $V\sim\text{Uniform}(A_i\cap A_C)$; the second step can be accomplished efficiently~\cite{chen2018fast}. Finally, for a random variable $X$ that is a function of $V$, we have the following identity:
\begin{align*}
\mathbb{E}_{V\sim P_{\V}}[X\mid C]=\mathbb{E}_{V\sim\tilde{P}_{\V}}\left[\frac{X\cdot f_{P_{\V}}(V)}{f_{\tilde{P}_{\V}}(V)}\;\bigg\vert\;C\right],
\end{align*}
where $f_{P_{\V}}$ and $f_{\tilde{P}_{\V}}$ are the density functions of $P_{\V}$ and $\tilde{P}_{\V}$, respectively, and where we assume that the support of $\tilde{P}_{\V}$ contains the support of $P_{\V}$. Thus, the importance sampling estimator is
\begin{align}
\label{eqn:importance}
\hat{\mu}_X=\sum_{i=1}^n\frac{X_i\cdot f_{P_{\V}}(V_i)}{f_{\tilde{P}_{\V}}(V_i)},
\end{align}
for samples $V_1,...,V_n$ and where the corresponding values of $X$ are $X_1,...,X_n$. Assuming
\begin{align*}
\operatorname*{\arg\max}_{V\in\V}\frac{X\cdot f_{P_{\V}}(V)}{f_{\tilde{P}_{\V}}(V)}\le1,
\end{align*}
then Theorem~\ref{thm:bernoulliconcentration} continues to hold for Eq.~\ref{eqn:importance}; in general, it is straightforward to scale $X$ so that the theorem applies.

One caveat is that this approach is that it requires that $P_{\V}$ has bounded domain (since the support of $\tilde{P}_{\V}$ must contain the support of $P_{\V}$). Technically, the same is true for \fairsquare; in particular, since tails of most distributions are small (e.g., Gaussian distributions have doubly exponentially decaying tails), truncating the distribution yields a very good approximation. However, the \fairsquare algorithm remains sound since its upper and lower bounds account for the error due to truncation; thus, it converges as long as the truncation error is smaller than the tolerance $\varepsilon$. Similarly, we can likely bound the error for our algorithm, but we leave this approach to future work.

\paragraph{{\bf\em Limitations.}}

As we have already discussed, our algorithm suffers from several limitations. Unlike \fairsquare, it is only able to provide high-probability fairness guarantees. Nevertheless, in practice, our experiments show that we can make the failure probability vanishingly small (e.g., $\Delta=10^{-10}$). Furthermore, our termination guarantee is not absolute, and there are inputs for which our algorithm would fail to terminate (i.e., where fairness ``just barely'' holds). However, existing tools such as \fairsquare would fail to terminate on these problem instances as well. Finally, our approach would have difficulty if the events conditioned on in the population model have very low probability since it relies on rejection sampling.

\paragraph{{\bf\em Challenges for specifications beyond fairness.}}

We focus on fairness properties since sampling population models tends to be very scalable in this setting. In particular, we find that sampling the population model is usually efficient---as above, they are often learned probabilistic models, which are designed to be easy to sample. Furthermore, we find that the conditional statements in the fairness specifications usually do not encode rare events---e.g., in the case of demographic parity, we do not expect minority and majority subpopulations to be particularly rare. In more general settings, there are often conditional sampling problems that are more challenging. For these settings, more sophisticated sampling algorithms would need to be developed, possibly along the lines of what we described above.

Furthermore, our specification language is tailored to fairness specifications, which typically consist of inequalities over arithmetic formulas, and boolean formulas over these inequalities. For other specifications, other kinds of logical operators such as temporal operators may be needed.

Finally, we note that our approach cannot be applied to verifying adversarial properties such as robustness~\cite{goodfellow2014explaining}, which inherently require solving an optimization problem over the inputs of the machine learning model. In contrast, fairness properties are probabilistic in the sense that they can be expressed as expectations over the outputs of the machine learning model.

\section{Related Work}

\paragraph{{\bf\em Verifying fairness.}}

The work most closely related to ours is~\cite{albarghouthi2017fairsquare}, which uses numerical integration to verify fairness properties of machine learning models including decision trees, SVMs, and neural networks. Because they rely on constraint solving techniques (in particular, SMT solvers), their tool is substantially less scalable than ours---whereas their tool does not even scale to a neural network with 37 parameters (including those in the Bayes net population model), our tool scales to deep neural networks with 16 million parameters. In contrast to their work, our algorithm may return an incorrect result; however, in our evaluation, we show that these events are very unlikely to happen.

\paragraph{{\bf\em Checking fairness using hypothesis testing.}}

There has also been recent work on checking whether fairness holds by using hypothesis testing~\cite{galhotra2017fairness}. There are two major advantages of our work compared to their approach. First, they use $p$-values, which are asymptotic, so they cannot give any formal guarantees; furthermore, they do not account for multiple hypothesis testing, which can yield misleading results. In contrast, our approach establishes concrete, high-probability fairness guarantees. Second, their approach is tailored to a single fairness definition. In contrast, our algorithm can be used with a variety of fairness specifications (including theirs).

\paragraph{{\bf\em Fairness in machine learning.}}

There has been a large literature attempting to devise new fairness specifications, including demographic parity~\cite{calders2009building}, equal opportunity~\cite{hardt2016equality}, and approaches based on causality~\cite{kusner2017counterfactual,kilbertus2017avoiding}. There has also been a large literature focusing on how to train fair machine learning classifiers~\cite{pedreshi2008discrimination,calders2010three,dwork2012fairness,fish2016confidence,corbett2017algorithmic,dwork2018decoupled} and transforming the data into fair representations~\cite{zemel2013learning,hajian2013methodology,feldman2015certifying,calmon2017optimized}. Finally, there has been work on quantifying the influence of input variables on the output of a machine learning classifier; this technique can be used to study fairness, but does not provide any formal fairness guarantees \cite{datta2017algorithmic}. In contrast, our work takes fairness properties as given, and aims to design algorithms for verifying the correctness of existing machine learning systems, which are treated as blackbox functions.

\paragraph{{\bf\em Verifying probabilistic properties.}}

There has been a long history of work attempting to verify probabilistic properties, including program analysis~\cite{sankaranarayanan2013static,sampson2014expressing,albarghouthi2017fairsquare}, symbolic execution~\cite{geldenhuys2012probabilistic,filieri2013reliability}, and model checking~\cite{younes2002probabilistic,grosu2005monte,clarke2011statistical,kwiatkowska2002prism}. Many of these tools rely on techniques such as numerical integration, which do not scale in our setting~\cite{albarghouthi2017fairsquare}. Alternatively, abstraction interpretation has been extended to probabilistic programs~\cite{monniaux2000abstract,monniaux2001abstract,monniaux2001backwards,claret2013bayesian}; see~\cite{gordon2014probabilistic} for a survey. However, these approaches may be imprecise and incomplete (even on non-pathelogical problem instances).

\paragraph{{\bf\em Statistical model checking.}}

There has been work on using statistical hypothesis tests to check probabilistic properties~\cite{younes2002probabilistic2,younes2002probabilistic,grosu2005monte,herault2006apmc,clarke2011statistical,sankaranarayanan2013static,sampson2014expressing}.

One line of work relies on a fixed sample size~\cite{herault2004approximate,sen2004statistical,sen2005statistical,sampson2014expressing}. Then, they use a statistical test to compute a bound on the probability that the property holds. Assuming a concentration inequality such as Hoeffding's inequality is used~\cite{herault2004approximate},
\footnote{We note that Hoeffding's inequality is sometimes called the Chernoff-Hoeffding inequality. It handles an additive error $|\hat{\mu}_Z-\mu_Z|\le\varepsilon$. The variant of the bound that handles multiplicative error $|\hat{\mu}_Z-\mu_Z|\le\varepsilon)\mu_Z$ is typically called Chernoff's inequality.}
then they can obtain high-probability bounds such as ours. A key drawback is that because they do not adaptively collect data, there is a chance that the statistical test will be able to neither prove nor disprove the specification. Furthermore, simply re-running the algorithm is not statistically sound, since it runs into the problem of multiple hypothesis testing~\cite{zhao2016adaptive,johari2017peeking}.

An alternative approach that has been studied is to leverage adaptive statistical hypothesis tests---in particular, Wald's sequential probability ratio test (SPRT)~\cite{wald1945sequential}. Like the adaptive concentration inequalities used in our work, SPRT continues to collect data until the specification is either proven or disproven~\cite{younes2002probabilistic2,younes2002probabilistic,younes2004verification,younes2006statistical,legay2010statistical}. SPRT can distinguish two hypotheses of the form
\begin{align*}
H_0\equiv\mu_Z\le d_0
\hspace{0.1in}\text{ vs. }\hspace{0.1in}
H_1\equiv\mu_Z\ge d_1,
\end{align*}
where $Z$ is a Bernoulli random variable and $d_1>d_0$. There are two key shortcomings of these approaches. First, we need to distinguish $H_0$ vs. $\neg H_0$ (or equivalently, the case $d_0=d_1$). This limitation is fundamental to approaches based on Wald's test---it computes a statistic $S_0$ based on $d_0$ and a statistic $S_1$ based on $d_1$, and compares them; if $d_0=d_1$, then we always have $S_0=S_1$, so the test can never distinguish $H_0$ from $H_1$. Second, Wald's test requires that the distribution of the random variables is known (but the parameters of the distribution may be unknown). While we have made this assumption (i.e., they are Bernoulli), our techniques are much more general. In particular, we only require a bound on the random variables. Indeed, our techniques directly apply to the setting where $\rmin=f(\vmin)$ and $\rmaj=f(\vmaj)$ are only known to satisfy $\rmin,\rmaj\in[0,1]$. In particular, Theorem~\ref{thm:bernoulliconcentration} applies as stated to random variables with domain $[0,1]$.

Finally, for verifying fairness properties, we need to compare a ratio of means $\frac{\mu_{\rmin}}{\mu_{\rmaj}}$ rather than a single mean $\mu_Z$. Prior work has focused on developing inference rules for handling formulas in temporal logics such as continuous stochastic logic (CSL)~\cite{younes2002probabilistic2,sen2004statistical} and linear temporal logic (LTL)~\cite{herault2004approximate} rather than arithmetic formulas such as ours. The inference rules we develop enable us to do so.

\paragraph{{\bf\em Verifying machine learning systems.}}

More broadly, there has been a large amount of recent work on verifying machine learning systems; the work has primarily focused on verifying robustness properties of deep neural networks~\cite{goodfellow2014explaining,bastani2016measuring,katz2017reluplex,huang2017safety,tjeng2017verifying,raghunathan2018certified,gehr2018safety}. At a high level, robustness can be thought of as an optimization problem (e.g., MAX-SMT), whereas fairness properties involve integration and are therefore more similar to counting problems (e.g., COUNTING-SMT). In general, counting is harder than optimization~\cite{valiant1979complexity}, at least when asking for exact solutions. In our setting, we can obtain high-probability approximations of the counts.

\section{Conclusion}

We have designed an algorithm for verifying fairness properties of machine learning systems. Our algorithm uses a sampling-based approach in conjunction with adaptive concentration inequalities to achieve probabilistic soundness and precision guarantees. As we have shown, our implementation \toolname can scale to large machine learning models, including a deep recurrent neural network benchmark that is more than six orders of magnitude larger than the largest neural network in the \fairsquare benchmark. While we have focused on verifying fairness, we believe that our approach of using adaptive concentration inequalities can be applied to verify other probabilistic properties as well.

%% file: results.tex
\begin{table*}
\caption{Results from comparing \toolname to \fairsquare~\cite{albarghouthi2017fairsquare}. For each problem instance (i.e., a classifier and population model), we show the total number of lines of code (LOC), the response of each tool, the running time of each tool (in seconds, timed out after 900 seconds), the ratio of the running time of \toolname to that of \fairsquare (lower is better), and for the rejection sampling strategy used by \toolname, the number of accepted samples, total samples, and the acceptance rate. In the ratio of running times, we conservatively assume \fairsquare takes 900 seconds to run if it times out; this ratio sometimes equals 0 due to rounding error.}
\label{tab:fairsquare}
\scriptsize
\begin{tabular}{llrrrrrrrrr}
\hline
\multirow{2}{*}{{\bf Classifier}}
& \multirow{2}{*}{{\bf\shortstack[c]{Pop.\\Model}}}
& \multirow{2}{*}{{\bf LOC}}
& \multicolumn{2}{c}{{\bf Is Fair?}}
& \multicolumn{3}{c}{{\bf Running Time (s)}}
& \multicolumn{3}{c}{{\bf Samples}}
\\
&&& \toolname
& \fairsquare
& \toolname
& \fairsquare
& Ratio
& Accepted
& Total
& Accept Rate
\\
\hline
$\text{DT}_{4}$ & Ind. & 17 & 1 & 1 & 21.2 & 2.1 & 9.9 & 91710 & 443975 & 20.7\% \\
$\text{DT}_{14}$ & Ind. & 34 & 1 & 1 & 120.4 & 4.1 & 29.3 & 365503 & 1768404 & 20.7\% \\
$\text{DT}_{16}$ & Ind. & 38 & 1 & 1 & 17.3 & 5.6 & 3.1 & 49095 & 236822 & 20.7\% \\
$\text{DT}_{16}^{\alpha}$ & Ind. & 42 & 1 & 1 & 3.1 & 6.4 & 0.5 & 7221 & 35377 & 20.4\% \\
$\text{DT}_{44}$ & Ind. & 95 & 1 & 1 & 33.3 & 19.5 & 1.7 & 68078 & 329859 & 20.6\% \\
$\text{SVM}_{3}$ & Ind. & 15 & 1 & 1 & 9.4 & 2.4 & 3.9 & 34304 & 166274 & 20.6\% \\
$\text{SVM}_{4}$ & Ind. & 17 & 1 & 1 & 9.6 & 3.5 & 2.7 & 33158 & 159964 & 20.7\% \\
$\text{SVM}_{4}^{\alpha}$ & Ind. & 19 & 1 & 1 & 1.7 & 3.0 & 0.6 & 5437 & 26013 & 20.9\% \\
$\text{SVM}_{5}$ & Ind. & 19 & 1 & 1 & 10.7 & 6.4 & 1.7 & 36315 & 175729 & 20.7\% \\
$\text{SVM}_{6}$ & Ind. & 21 & 1 & 1 & 7.8 & 5.4 & 1.4 & 28140 & 136722 & 20.6\% \\
$\text{NN}_{2,1}$ & Ind. & 22 & 1 & 1 & 2.3 & 3.9 & 0.6 & 9364 & 45289 & 20.7\% \\
$\text{NN}_{2,2}$ & Ind. & 25 & 1 & 1 & 2.9 & 6.1 & 0.5 & 11407 & 55102 & 20.7\% \\
$\text{NN}_{3,2}$ & Ind. & 27 & 1 & 1 & 6.4 & 435.6 & 0.0 & 20856 & 100855 & 20.7\% \\
$\text{DT}_{4}$ & B.N. 1 & 27 & 0 & 0 & 1.6 & 3.5 & 0.5 & 6208 & 29689 & 20.9\% \\
$\text{DT}_{14}$ & B.N. 1 & 48 & 1 & 1 & 156.0 & 21.8 & 7.1 & 442872 & 2147170 & 20.6\% \\
$\text{DT}_{16}$ & B.N. 1 & 51 & 0 & 0 & 2.4 & 15.3 & 0.2 & 5698 & 27422 & 20.8\% \\
$\text{DT}_{16}^{\alpha}$ & B.N. 1 & 55 & 1 & 1 & 24.4 & 27.7 & 0.9 & 64691 & 313671 & 20.6\% \\
$\text{DT}_{44}$ & B.N. 1 & 111 & 0 & 0 & 17.5 & 353.2 & 0.0 & 33750 & 163661 & 20.6\% \\
$\text{SVM}_{3}$ & B.N. 1 & 25 & 0 & 0 & 3.0 & 4.0 & 0.7 & 10347 & 49845 & 20.8\% \\
$\text{SVM}_{4}$ & B.N. 1 & 30 & 0 & 0 & 4.6 & 5.8 & 0.8 & 15009 & 72556 & 20.7\% \\
$\text{SVM}_{4}^{\alpha}$ & B.N. 1 & 32 & 1 & 1 & 5.2 & 10.4 & 0.5 & 16846 & 81355 & 20.7\% \\
$\text{SVM}_{5}$ & B.N. 1 & 35 & 0 & 0 & 3.5 & 11.1 & 0.3 & 12116 & 58197 & 20.8\% \\
$\text{SVM}_{6}$ & B.N. 1 & 40 & 0 & 0 & 3.0 & 19.0 & 0.2 & 9193 & 44575 & 20.6\% \\
$\text{NN}_{2,1}$ & B.N. 1 & 36 & 1 & 1 & 2.9 & 57.0 & 0.1 & 10345 & 50183 & 20.6\% \\
$\text{NN}_{2,2}$ & B.N. 1 & 39 & 1 & 1 & 4.8 & 32.7 & 0.1 & 14449 & 69779 & 20.7\% \\
$\text{NN}_{3,2}$ & B.N. 1 & 40 & 1 & T.O. & 88.3 & T.O. & 0.1 & 308228 & 1489839 & 20.7\% \\
$\text{DT}_{4}$ & B.N. 2 & 33 & 0 & 0 & 1.4 & 5.8 & 0.2 & 4790 & 23232 & 20.6\% \\
$\text{DT}_{14}$ & B.N. 2 & 54 & 1 & T.O. & 190.1 & T.O. & 0.2 & 524166 & 2535812 & 20.7\% \\
$\text{DT}_{16}$ & B.N. 2 & 57 & 0 & 0 & 3.1 & 35.4 & 0.1 & 7002 & 34194 & 20.5\% \\
$\text{DT}_{16}^{\alpha}$ & B.N. 2 & 61 & 1 & 1 & 24.0 & 60.0 & 0.4 & 61027 & 295445 & 20.7\% \\
$\text{DT}_{44}$ & B.N. 2 & 117 & 0 & T.O. & 22.1 & T.O. & 0.0 & 40841 & 197689 & 20.7\% \\
$\text{SVM}_{3}$ & B.N. 2 & 31 & 0 & 0 & 4.3 & 8.7 & 0.5 & 14392 & 69596 & 20.7\% \\
$\text{SVM}_{4}$ & B.N. 2 & 36 & 0 & 0 & 3.8 & 24.2 & 0.2 & 11113 & 53831 & 20.6\% \\
$\text{SVM}_{4}^{\alpha}$ & B.N. 2 & 38 & 1 & 1 & 5.9 & 22.1 & 0.3 & 18664 & 89394 & 20.9\% \\
$\text{SVM}_{5}$ & B.N. 2 & 41 & 0 & 0 & 3.8 & 496.7 & 0.0 & 12147 & 58115 & 20.9\% \\
$\text{SVM}_{6}$ & B.N. 2 & 42 & 0 & 0 & 3.9 & 87.8 & 0.0 & 11765 & 56820 & 20.7\% \\
$\text{NN}_{2,1}$ & B.N. 2 & 38 & 1 & 1 & 2.9 & 52.2 & 0.1 & 9717 & 47162 & 20.6\% \\
$\text{NN}_{2,2}$ & B.N. 2 & 41 & 1 & 1 & 4.1 & 126.4 & 0.0 & 12729 & 61965 & 20.5\% \\
$\text{NN}_{3,2}$ & B.N. 2 & 42 & 1 & T.O. & 110.9 & T.O. & 0.1 & 387860 & 1880146 & 20.6\% \\
\hline
\end{tabular}
\end{table*}

%% file: appendix.tex
\section{Proofs of Theoretical Results}

We prove a number of correctness results for Algorithm~\ref{alg:verify}.

\subsection{Proof of Theorem~\ref{thm:bernoulliconcentration}}
\label{sec:bernoulliconcentrationproof}

First, we have the following well-known definition, which is a key component for the adaptive concentration inequality we use~\cite{zhao2016adaptive}.
\begin{definition}
A random variable $Z$ is $d$-\emph{subgaussian} if $\mu_Z=0$ and
\begin{align*}
  \mathbb{E}[e^{rZ}]\le e^{d^2r^2/2}
\end{align*}
for all $r\in\mathbb{R}$.
\end{definition}
\noindent
\begin{theorem}
\label{thm:concentration}
Suppose that $Z$ is a $\frac{1}{2}$-subgaussian random variable with probability distribution $P_{\Z}$. Let
\begin{align*}
\hat{\mu}_Z^{(n)}=\frac{1}{n}\sum_{i=1}^nZ_i,
\end{align*}
where $\{Z_i\sim P_{\Z}\}_{i\in\mathbb{N}}$ are i.i.d. samples from $P_{\Z}$, let $J$ be a random variable on $\mathbb{N}\cup\{\infty\}$, let
\begin{align*}
\varepsilon_b(n)=\sqrt{\frac{\frac{3}{5}\cdot\log(\log_{11/10}n+1)+b}{n}}
\end{align*}
for some constant $b\in\mathbb{R}$, and let $\delta_b=24e^{-9b/5}$. Then,
\begin{align*}
\text{Pr}[J<\infty\wedge(|\hat{\mu}_Z^{(J)}|\ge\varepsilon_b(J))]\le\delta_b.
\end{align*}
\end{theorem}
Using this result, we first prove the following slight variant of Theorem~\ref{thm:bernoulliconcentration}, which accounts for the case $\text{Pr}[J<\infty]<1$.
\begin{theorem}
\label{thm:refinedbernoulliconcentration}
Given a Bernoulli random variable $Z$ with probability distribution $P_{\Z}$, let $\{Z_i\sim P_{\Z}\}_{i\in\mathbb{N}}$ be i.i.d. samples of $Z$, let
\begin{align*}
\hat{\mu}_Z^{(n)}=\frac{1}{n}\sum_{i=1}^nZ_i,
\end{align*}
let $J$ be a random variable on $\mathbb{N}\cup\{\infty\}$, and let
\begin{align*}
  \varepsilon(\delta,n)=\sqrt{\frac{\frac{3}{5}\cdot\log(\log_{11/10}n+1)+\frac{5}{9}\cdot\log(24/\delta)}{n}}
\end{align*}
for a given $\delta\in\mathbb{R}_+$. Then,
\begin{align*}
\text{Pr}[J<\infty\wedge(|\hat{\mu}_Z^{(J)}-\mu_Z|\ge\varepsilon(\delta,J))]\le\delta.
\end{align*}
\end{theorem}
\begin{proof}
  As described in~\cite{zhao2016adaptive}, any distribution bounded in an interval of length $2d$ is $d$-subgaussian. Thus, for any Bernoulli random variable $Z$, the random variable $Z-\mu_Z$ is $\frac{1}{2}$-subgassian. Then, the claim follows by applying Theorem~\ref{thm:concentration} (noting that $b=\frac{5}{9}\cdot\log(24/\delta_b)$).
\end{proof}

Note that Theorem~\ref{thm:bernoulliconcentration} follows immediately from Theorem~\ref{thm:refinedbernoulliconcentration} since it assumes that $\text{Pr}[J<\infty]=1$, so this term can be dropped from the probability event.
\qed

\subsection{Proof of Theorem~\ref{thm:sound}}
\label{sec:soundproof}

We prove by structural induction on the derivation.

\paragraph{Random variable.}

This case follows by our assumption that the initial environment $\Gamma$ is correct.

\paragraph{Constant.}

This case follows by definition since a constant $c$ satisfies $\llbracket c\rrbracket=c$.

\paragraph{Sum.}

By assumption, $|E-\llbracket X\rrbracket|\le\varepsilon$ with probability at least $1-\delta$, and $|E'-\llbracket X'\rrbracket|\le\varepsilon'$ with probability at least $1-\delta'$. By a union bound, both of these hold with probability at least $1-(\delta+\delta')$. Then,
\begin{align*}
|(E+E')-\llbracket X+X'\rrbracket|
&=|(E+E')-(\llbracket X\rrbracket+\llbracket X'\rrbracket)| \\
&\le|E-\llbracket X\rrbracket|+|E'-\llbracket X'\rrbracket| \\
&\le\varepsilon+\varepsilon'.
\end{align*}
In other words, we can conclude that $X+X':(E+E',\varepsilon+\varepsilon',\delta+\delta')$.

\paragraph{Negative.}

By assumption, $|E-\llbracket X\rrbracket|\le\varepsilon$ with probability at least $1-\delta$. Then,
\begin{align*}
|(-E)-\llbracket-X\rrbracket)|=|E-\llbracket X\rrbracket|\le\varepsilon
\end{align*}
In other words, we can conclude that $-X:(-E,\varepsilon,\delta)$.

\paragraph{Product.}

By assumption, $|E-\llbracket X\rrbracket|\le\varepsilon$ with probability at least $1-\delta$, and $|E'-\llbracket X'\rrbracket|\le\varepsilon'$ with probability at least $1-\delta'$. a union bound, both of these hold with probability at least $1-(\delta+\delta')$. Then,
\begin{align*}
|E'-E'+\llbracket X'\rrbracket|
&=|E'-E'+\llbracket X'\rrbracket| \\
&\le|E'|+|-E'+\llbracket X'\rrbracket| \\
&\le|E'|+\varepsilon',
\end{align*}
so
\begin{align*}
&|E\cdot E'-\llbracket X\cdot X'\rrbracket| \\
&=|E\cdot E'-\llbracket X\rrbracket\cdot\llbracket X'\rrbracket| \\
&=|E\cdot E'-E\cdot\llbracket X'\rrbracket+E\cdot\llbracket X'\rrbracket-\llbracket X\rrbracket\cdot\llbracket X'\rrbracket| \\
&=|E\cdot(E'-\llbracket X'\rrbracket)+\llbracket X'\rrbracket\cdot(E-\llbracket X\rrbracket)| \\
&\le|E|\cdot|E'-\llbracket X'\rrbracket|+|\llbracket X'\rrbracket|\cdot|E-\llbracket X\rrbracket| \\
&\le|E|\cdot\varepsilon'+|\llbracket X'\rrbracket|\cdot\varepsilon \\
&\le|E|\cdot\varepsilon'+(|E'|+\varepsilon')\cdot\varepsilon \\
&=|E|\cdot\varepsilon'+|E'|\cdot\varepsilon+\varepsilon\cdot\varepsilon'.
\end{align*}
In other words, we can conclude that $X\cdot X':(E\cdot E',E\cdot\varepsilon'+E'\cdot\varepsilon+\varepsilon\cdot\varepsilon',\delta+\delta')$.

\paragraph{Inverse.}

By assumption, $|E-\llbracket X\rrbracket|\le\varepsilon$ with probability at least $1-\delta$. Then,
\begin{align*}
|E|
&=|E-\llbracket X\rrbracket+\llbracket X\rrbracket| \\
&\le|E-\llbracket X\rrbracket|+|\llbracket X\rrbracket| \\
&\le\varepsilon+|\llbracket X\rrbracket|,
\end{align*}
i.e., $|\llbracket X\rrbracket|\ge|E|-\varepsilon$, so
\begin{align*}
|E^{-1}-\llbracket X^{-1}\rrbracket|
&=|E^{-1}-\llbracket X\rrbracket^{-1}| \\
&=\left|\frac{\llbracket X\rrbracket-E}{E\cdot\llbracket X\rrbracket}\right| \\
&\le\frac{\varepsilon}{|E|\cdot|\llbracket X\rrbracket|} \\
&\le\frac{\varepsilon}{|E|\cdot(|E|-\varepsilon)},
\end{align*}
where the last step follows since we have assumed that $|E|-\varepsilon>0$. In other words, we can conclude that $X^{-1}:(E^{-1},\frac{\varepsilon}{|E|\cdot(|E|-\varepsilon)},\delta)$.

\paragraph{Inequality true.}

By assumption, $|E-\llbracket X\rrbracket|\le\varepsilon$ with probability at least $1-\delta$, and furthermore $E-\varepsilon\ge0$. Thus,
\begin{align*}
E-\llbracket X\rrbracket\le\varepsilon,
\end{align*}
or equivalently,
\begin{align*}
\llbracket X\rrbracket\ge E-\varepsilon\ge0.
\end{align*}
In other words, we can conclude that $X\ge0:(\true,\delta)$.

\paragraph{Inequality false.}

By assumption, $|E-\llbracket X\rrbracket|\le\varepsilon$ with probability at least $1-\delta$, and furthermore $E+\varepsilon<$. Thus,
\begin{align*}
\llbracket X\rrbracket-E\le\varepsilon,
\end{align*}
or equivalently,
\begin{align*}
\llbracket X\rrbracket\le E+\varepsilon<0.
\end{align*}
In other words, we can conclude that $X\ge0:(\false,\delta)$.

\paragraph{And.}

By assumption, $\llbracket Y\rrbracket=I$ with probability at least $1-\delta$, and $\llbracket Y'\rrbracket=I'$ with probability at least $1-\delta'$. a union bound, both of these hold with probability at least $1-(\delta+\delta')$. Then,
\begin{align*}
\llbracket Y\wedge Y'\rrbracket=\llbracket Y\rrbracket\wedge\llbracket Y'\rrbracket=I\wedge I'.
\end{align*}
In other words, we can conclude that $Y\wedge Y':(I\wedge I',\delta+\delta')$.

\paragraph{Or.}

By assumption, $\llbracket Y\rrbracket=I$ with probability at least $1-\delta$, and $\llbracket Y'\rrbracket=I'$ with probability at least $1-\delta'$. a union bound, both of these hold with probability at least $1-(\delta+\delta')$. Then,
\begin{align*}
\llbracket Y\vee Y'\rrbracket=\llbracket Y\rrbracket\vee\llbracket Y'\rrbracket=I\vee I'.
\end{align*}
In other words, we can conclude that $Y\vee Y':(I\vee I',\delta+\delta')$.

\paragraph{Not.}

By assumption, $\llbracket Y\rrbracket=I$ with probability at least $1-\delta$. Then,
\begin{align*}
\llbracket\neg Y\rrbracket=\neg\llbracket Y\rrbracket=\neg I.
\end{align*}
In other words, we can conclude that $\neg Y:(\neg I,\delta)$.
\qed

\subsection{Proof of Theorem~\ref{thm:inferenceterminate}}
\label{sec:inferenceterminateproof}

First, we prove the following stronger lemma, which says that as $n\to\infty$ (where $n$ is the number of samples), our algorithm eventually infers arbitrarily tight bounds on any given well-defined problem instance. Then, Theorem~\ref{thm:inferenceterminate} follows from the applying this lemma to the given specification $Y$ and $\gamma=\Delta$, where $\Delta\in\mathbb{R}_+$ is the given confidence level.
\begin{lemma}
\label{lem:bound}
Given any well-defined problem instance $(P_{\Z},X)$, where $X\in\mathcal{L}(T)$, and any $\delta\in\mathbb{R}_+$, let
\begin{align*}
  \Gamma^{(n)}=\{\mu_Z:(E^{(n)},\varepsilon(\delta_Z,n),\delta_Z)\}
\end{align*}
where
\begin{align*}
  E^{(n)}&=\frac{1}{n}\sum_{i=1}^nZ_i \\
  \varepsilon(\delta_Z,n)&=\sqrt{\frac{\frac{3}{5}\cdot\log(\log_{11/10}n+1)+\frac{5}{9}\cdot\log(24/\delta_Z)}{n}} \\
  \delta_Z&=\delta/\llbracket X\rrbracket_{\delta}.
\end{align*}
Intuitively, $\Gamma^{(n)}$ is the lemma established for $\mu_Z$ on the $n$th iteration of Algorithm~\ref{alg:verify}. Then, for any $\varepsilon\in\mathbb{R}_+$ and any $\varepsilon_0,\delta_0\in\mathbb{R}_+$, there exists $n_0\in\mathbb{N}$ such that for any $n\ge n_0$, with probability at least $1-\delta_0$, so
\begin{align*}
  \Gamma^{(n)}\vdash X:(E,\varepsilon,\delta)
\end{align*}
for some $E\in\mathbb{R}$ such that $|E-\llbracket X\rrbracket|\le\varepsilon_0$. We are allowed to make the given values $\varepsilon,\delta,\varepsilon_0,\delta_0$ smaller.

Similarly, given any well-defined problem-instance $(P_{\Z},Y)$, where $Y\in\mathcal{L}(S)$ and any $\gamma\in\mathbb{R}_+$, let
\begin{align*}
  \Gamma^{(n)}=\{\mu_Z:(E^{(n)},\varepsilon(\delta_Z,n),\delta_Z)\}
\end{align*}
as before. Then, for any $\gamma\in\mathbb{R}_+$ and $\delta_0\in\mathbb{R}_+$, there exists $n_0\in\mathbb{N}$ such that for all $n\ge n_0$, with probability at least $1-\delta_0$, so
\begin{align*}
  \Gamma^{(n)}\vdash Y:(\llbracket Y\rrbracket,\gamma).
\end{align*}
Again, we are allowed to make the given values $\gamma,\delta_0$ smaller.
\end{lemma}

\begin{proof}
We prove by structural induction on the inference rules in Figure~\ref{fig:inference}, focusing on the following cases of interest: random variables, inverses, and inequalities; the remaining cases follow similarly.

\paragraph{{\bf\em Random variable.}}

Consider the specification $\mu_Z$, and let $\varepsilon,\delta,\varepsilon_0,\delta_0\in\mathbb{R}_+$ be given. Note that as $n\to\infty$, we have $\varepsilon(\delta_Z,n)\to0$; furthermore, $\delta_Z=\delta/\llbracket\mu_Z\rrbracket_{\delta}=\delta$. Thus, it suffices to prove that as $E^{(n)}\to\mu_Z$ as $n\to\infty$ as well. To this end, let
\begin{align*}
  n_0=\frac{\log(2/\delta_0)}{2(\varepsilon_0)^2}.
\end{align*}
By Hoeffding's inequality,
\begin{align*}
  \text{Pr}[|E^{(n)}-\mu_Z|\le\varepsilon_0]\ge1-2e^{-2n\varepsilon_0^2}\ge1-2e^{-2n_0\varepsilon_0^2}=1-\delta_0,
\end{align*}
as claimed.

\paragraph{{\bf\em Inverse.}}

Consider the specification $X^{-1}$, and let $\varepsilon,\delta,\varepsilon_0,\delta_0\in\mathbb{R}^+$ be given. Because we have assumed that the problem instance is well-defined, we must have $\llbracket X\rrbracket\neq0$. Let $\alpha=|\llbracket X\rrbracket|$, and let
\begin{align*}
  \tilde{\varepsilon}&=\min\left\{\frac{\varepsilon\cdot(\alpha/2)^2}{1+\varepsilon\cdot(\alpha/2)},~\frac{\alpha}{2}\right\} \\
  \tilde{\delta}&=\delta \\
  \tilde{\varepsilon}_0&=\min\left\{\frac{\alpha}{4},~\frac{\varepsilon_0\cdot\alpha^2}{2}\right\} \\
  \tilde{\delta}_0&=\delta_0.
\end{align*}
Note that $\delta_Z=\delta/\llbracket X^{-1}\rrbracket_{\delta}=\tilde{\delta}/\llbracket X\rrbracket_{\delta}$. Therefore, by induction, there exists $n_0\in\mathbb{N}$ such that for all $n\ge n_0$, with probability at least $1-\tilde{\delta}_0=1-\delta_0$, our algorithm proves the lemma
\begin{align*}
  \Gamma^{(n)}\vdash X:(\tilde{E},\tilde{\varepsilon},\tilde{\delta}),
\end{align*}
where $|\tilde{E}-\llbracket X\rrbracket||\le\tilde{\varepsilon}_0$. Then, note that
\begin{align*}
  \frac{\alpha}{2}>\tilde{\varepsilon}_0
  &\ge|\tilde{E}-\llbracket X\rrbracket| \\
  &\ge|\llbracket X\rrbracket|-|\tilde{E}| \\
  &\ge\alpha-|\tilde{E}|,
\end{align*}
from which it follows that
\begin{align*}
  |\tilde{E}|>\frac{\alpha}{2}\ge\tilde{\varepsilon}.
\end{align*}
Thus, the inference rule for inverses applies, so Algorithm~\ref{alg:verify} proves the lemma
\begin{align*}
  \Gamma^{(n)}\vdash X^{-1}:\left(\tilde{E}^{-1},\frac{\tilde{\varepsilon}}{|\tilde{E}|(|\tilde{E}|-\tilde{\varepsilon})},\tilde{\delta}\right).
\end{align*}
Next, note that
\begin{align*}
  \tilde{\varepsilon}\le\frac{\varepsilon\cdot(\alpha/2)^2}{1+\varepsilon\cdot(\alpha/2)}\le\frac{\varepsilon\cdot(\alpha/2)\cdot|\tilde{E}|}{1+\varepsilon\cdot(\alpha/2)},
\end{align*}
from which it follows that
\begin{align*}
  \varepsilon\ge\frac{\tilde{\varepsilon}}{(\alpha/2)\cdot(|\tilde{E}|-\tilde{\varepsilon})}\ge\frac{\tilde{\varepsilon}}{|\tilde{E}|(|\tilde{E}|-\tilde{\varepsilon})}.
\end{align*}
Furthermore, we have $\tilde{\delta}\le\delta$. Finally, note that
\begin{align*}
  |\tilde{E}^{-1}-\llbracket X\rrbracket^{-1}|
  &=\left|\frac{\tilde{E}-\llbracket X\rrbracket}{\tilde{E}\cdot\llbracket X\rrbracket}\right| \\
  &\le\frac{\tilde{\varepsilon}_0}{\alpha^2/2} \\
  &\le\varepsilon_0,
\end{align*}
which holds with probability at least $\delta_0\le\tilde{\delta_0}$. Note that we can make $\varepsilon$ and $\varepsilon_0$ smaller so that
\begin{align*}
  \Gamma^{(n)}\vdash X^{-1}:(E,\varepsilon,\delta),
\end{align*}
where $E=\tilde{E}^{-1}$ satisfies $|E-\llbracket X^{-1}\rrbracket|\le\varepsilon_0$, so the claim follows.

\paragraph{{\bf\em Inequality.}}

Consider the specification $X\ge0$, and let $\gamma,\delta_0\in\mathbb{R}_+$ be given. Let $\alpha=|\llbracket X\rrbracket|$, and let
\begin{align*}
  \tilde{\varepsilon}&=\frac{\alpha}{3} \\
  \tilde{\delta}&=\gamma \\
  \tilde{\varepsilon}_0&=\frac{\alpha}{3} \\
  \tilde{\delta}_0&=\delta_0.
\end{align*}
Note that $\delta_Z=\gamma/\llbracket X\ge0\rrbracket_{\delta}=\tilde{\delta}/\llbracket X\rrbracket_{\delta}$. Therefore, by induction, there exists $n_0\in\mathbb{N}$ such that for any $n\ge n_0$, with probability at least $1-\tilde{\delta}_0=1-\delta_0$, our algorithm proves the lemma
\begin{align*}
  \Gamma^{(n)}\vdash X:(\tilde{E},\tilde{\varepsilon},\tilde{\delta}),
\end{align*}
where $|\tilde{E}-\llbracket X\rrbracket|\le\tilde{\varepsilon}_0$. Without loss of generality, assume that $\llbracket X\rrbracket\ge0$ (so $\alpha=\llbracket X\rrbracket$). Then, note that
\begin{align*}
  \tilde{E}
  &\ge\llbracket X\rrbracket-\tilde{\varepsilon} \\
  &\ge\frac{2\alpha}{3} \\
  &>\tilde{\varepsilon},
\end{align*}
so $\tilde{E}-\tilde{\varepsilon}\ge0$, which implies that the inference rule for true inequalities applies. Thus, our algorithm proves the lemma
\begin{align*}
  \Gamma^{(n)}\vdash X:(\true,\tilde{\delta}),
\end{align*}
where $\tilde{\delta}=\gamma$. Note that since $\llbracket X\rrbracket\ge0$, we have $\llbracket X\ge0\rrbracket=\true$, so the claim follows.
\end{proof}

\subsection{Proof of Theorem~\ref{lem:terminate}}
\label{sec:terminateproof}

To show that Algorithm~\ref{alg:verify} terminates with probability $1$, it suffices to show that for any $\delta_0\in\mathbb{R}$, there exists $n_0\in\mathbb{N}$ such that our algorithm terminates after $n\le n_0$ steps with probability at least $1-\delta_0$. Applying Lemma~\ref{lem:bound}, we have that there exists $n_0\in\mathbb{N}$ such that with probability at least $1-\delta_0$, so
\begin{align*}
  \Gamma^{(n)}\vdash Y:(\llbracket Y\rrbracket,\gamma),
\end{align*}
where $\gamma\le\Delta$, where $\Delta$ is the confidence level given as input to Algorithm~\ref{alg:verify}. Therefore, the claim follows.
\qed

\subsection{Proof of Theorem~\ref{thm:terminate}}
\label{sec:mainproof}

For simplicity, we consider the case where there is a single leaf node labeled $\mu_Z$ in the given specification $Y$ (so $\llbracket Y\rrbracket_{\delta}=1$); the general case is a straightforward extension. First, we claim that if Algorithm~\ref{alg:verify} terminates and returns an incorrect response, then it must be the case that
\begin{align*}
  |\hat{\mu}_Z^{(J)}-\mu_Z|>\varepsilon(\delta_Z,J),
\end{align*}
where
\begin{align*}
\delta_Z=\Delta/\llbracket Y\rrbracket_{\delta}=\Delta,
\end{align*}
and $J$ is the number of iterations of our algorithm. Suppose to the contrary; then, the lemma
\begin{align*}
\mu_Z:(s/J,\varepsilon_Z(s/n,J),\delta_Z)
\end{align*}
in $\Gamma$ on the $J$th iteration of our algorithm holds. By Theorem~\ref{thm:sound}, we have $\Gamma\vdash Y:(I,\gamma)$ if and only if
\begin{align*}
\text{Pr}[\llbracket Y\rrbracket=I]\ge1-\gamma.
\end{align*}
Since Algorithm~\ref{alg:verify} has terminated, then it must be the case that $\gamma\le\Delta$. Thus, the response is correct, which is a contradiction, so the claim follows. Then,
\begin{align*}
  &\text{Pr}[\text{Algorithm}~\ref{alg:verify}~\text{terminates and responds incorrectly}] \\
  &\le\text{Pr}[J<\infty\wedge|\hat{\mu}_Z^{(J)}-\mu_Z|>\varepsilon(\delta_Z,J)] \\
  &\le\delta_Z \\
  &\le\Delta.
\end{align*}
The second inequality follows from Theorem~\ref{thm:refinedbernoulliconcentration}. Thus, Algorithm~\ref{alg:verify} is probabilistically sound and precise, as claimed.
\qed

%% file: paper.bbl

\begin{thebibliography}{66}

%%% the Web address.

\ifx \showCODEN    \undefined \def \showCODEN     #1{\unskip}     \fi
\ifx \showDOI      \undefined \def \showDOI       #1{#1}\fi
\ifx \showISBNx    \undefined \def \showISBNx     #1{\unskip}     \fi
\ifx \showISBNxiii \undefined \def \showISBNxiii  #1{\unskip}     \fi
\ifx \showISSN     \undefined \def \showISSN      #1{\unskip}     \fi
\ifx \showLCCN     \undefined \def \showLCCN      #1{\unskip}     \fi
\ifx \shownote     \undefined \def \shownote      #1{#1}          \fi
\ifx \showarticletitle \undefined \def \showarticletitle #1{#1}   \fi
\ifx \showURL      \undefined \def \showURL       {\relax}        \fi
% invisible to TeX
\providecommand\bibfield[2]{#2}
\providecommand\bibinfo[2]{#2}
\providecommand\natexlab[1]{#1}
\providecommand\showeprint[2][]{arXiv:#2}

\bibitem[\protect\citeauthoryear{Albarghouthi, D'Antoni, Drews, and
  Nori}{Albarghouthi et~al\mbox{.}}{2017}]%
        {albarghouthi2017fairsquare}
\bibfield{author}{\bibinfo{person}{Aws Albarghouthi}, \bibinfo{person}{Loris
  D'Antoni}, \bibinfo{person}{Samuel Drews}, {and} \bibinfo{person}{Aditya~V
  Nori}.} \bibinfo{year}{2017}\natexlab{}.
\newblock \showarticletitle{FairSquare: probabilistic verification of program
  fairness}. In \bibinfo{booktitle}{\emph{OOPSLA}}.
\newblock


\bibitem[\protect\citeauthoryear{Barocas and Selbst}{Barocas and
  Selbst}{2016}]%
        {barocas2016big}
\bibfield{author}{\bibinfo{person}{Solon Barocas} {and}
  \bibinfo{person}{Andrew~D Selbst}.} \bibinfo{year}{2016}\natexlab{}.
\newblock \showarticletitle{Big data's disparate impact}.
\newblock \bibinfo{journal}{\emph{Cal. L. Rev.}}  \bibinfo{volume}{104}
  (\bibinfo{year}{2016}), \bibinfo{pages}{671}.
\newblock


\bibitem[\protect\citeauthoryear{Bastani, Ioannou, Lampropoulos, Vytiniotis,
  Nori, and Criminisi}{Bastani et~al\mbox{.}}{2016}]%
        {bastani2016measuring}
\bibfield{author}{\bibinfo{person}{Osbert Bastani}, \bibinfo{person}{Yani
  Ioannou}, \bibinfo{person}{Leonidas Lampropoulos}, \bibinfo{person}{Dimitrios
  Vytiniotis}, \bibinfo{person}{Aditya Nori}, {and} \bibinfo{person}{Antonio
  Criminisi}.} \bibinfo{year}{2016}\natexlab{}.
\newblock \showarticletitle{Measuring neural net robustness with constraints}.
  In \bibinfo{booktitle}{\emph{Advances in neural information processing
  systems}}. \bibinfo{pages}{2613--2621}.
\newblock


\bibitem[\protect\citeauthoryear{Biddle}{Biddle}{2006}]%
        {biddle2006adverse}
\bibfield{author}{\bibinfo{person}{Dan Biddle}.}
  \bibinfo{year}{2006}\natexlab{}.
\newblock \bibinfo{booktitle}{\emph{Adverse impact and test validation: A
  practitioner's guide to valid and defensible employment testing}}.
\newblock \bibinfo{publisher}{Gower Publishing, Ltd.}
\newblock


\bibitem[\protect\citeauthoryear{Calders, Kamiran, and Pechenizkiy}{Calders
  et~al\mbox{.}}{2009}]%
        {calders2009building}
\bibfield{author}{\bibinfo{person}{Toon Calders}, \bibinfo{person}{Faisal
  Kamiran}, {and} \bibinfo{person}{Mykola Pechenizkiy}.}
  \bibinfo{year}{2009}\natexlab{}.
\newblock \showarticletitle{Building classifiers with independency
  constraints}. In \bibinfo{booktitle}{\emph{ICDMW}}. \bibinfo{pages}{13--18}.
\newblock


\bibitem[\protect\citeauthoryear{Calders and Verwer}{Calders and
  Verwer}{2010}]%
        {calders2010three}
\bibfield{author}{\bibinfo{person}{Toon Calders} {and} \bibinfo{person}{Sicco
  Verwer}.} \bibinfo{year}{2010}\natexlab{}.
\newblock \showarticletitle{Three naive Bayes approaches for
  discrimination-free classification}.
\newblock \bibinfo{journal}{\emph{Data Mining and Knowledge Discovery}}
  \bibinfo{volume}{21}, \bibinfo{number}{2} (\bibinfo{year}{2010}),
  \bibinfo{pages}{277--292}.
\newblock


\bibitem[\protect\citeauthoryear{Calmon, Wei, Vinzamuri, Ramamurthy, and
  Varshney}{Calmon et~al\mbox{.}}{2017}]%
        {calmon2017optimized}
\bibfield{author}{\bibinfo{person}{Flavio Calmon}, \bibinfo{person}{Dennis
  Wei}, \bibinfo{person}{Bhanukiran Vinzamuri},
  \bibinfo{person}{Karthikeyan~Natesan Ramamurthy}, {and}
  \bibinfo{person}{Kush~R Varshney}.} \bibinfo{year}{2017}\natexlab{}.
\newblock \showarticletitle{Optimized Pre-Processing for Discrimination
  Prevention}. In \bibinfo{booktitle}{\emph{Advances in Neural Information
  Processing Systems}}. \bibinfo{pages}{3995--4004}.
\newblock


\bibitem[\protect\citeauthoryear{Chen, Dwivedi, Wainwright, and Yu}{Chen
  et~al\mbox{.}}{2018}]%
        {chen2018fast}
\bibfield{author}{\bibinfo{person}{Yuansi Chen}, \bibinfo{person}{Raaz
  Dwivedi}, \bibinfo{person}{Martin~J Wainwright}, {and} \bibinfo{person}{Bin
  Yu}.} \bibinfo{year}{2018}\natexlab{}.
\newblock \showarticletitle{Fast MCMC sampling algorithms on polytopes}.
\newblock \bibinfo{journal}{\emph{The Journal of Machine Learning Research}}
  \bibinfo{volume}{19}, \bibinfo{number}{1} (\bibinfo{year}{2018}),
  \bibinfo{pages}{2146--2231}.
\newblock


\bibitem[\protect\citeauthoryear{Claret, Rajamani, Nori, Gordon, and
  Borgstr{\"o}m}{Claret et~al\mbox{.}}{2013}]%
        {claret2013bayesian}
\bibfield{author}{\bibinfo{person}{Guillaume Claret}, \bibinfo{person}{Sriram~K
  Rajamani}, \bibinfo{person}{Aditya~V Nori}, \bibinfo{person}{Andrew~D
  Gordon}, {and} \bibinfo{person}{Johannes Borgstr{\"o}m}.}
  \bibinfo{year}{2013}\natexlab{}.
\newblock \showarticletitle{Bayesian inference using data flow analysis}. In
  \bibinfo{booktitle}{\emph{Proceedings of the 2013 9th Joint Meeting on
  Foundations of Software Engineering}}. ACM, \bibinfo{pages}{92--102}.
\newblock


\bibitem[\protect\citeauthoryear{Clarke and Zuliani}{Clarke and
  Zuliani}{2011}]%
        {clarke2011statistical}
\bibfield{author}{\bibinfo{person}{Edmund~M Clarke} {and}
  \bibinfo{person}{Paolo Zuliani}.} \bibinfo{year}{2011}\natexlab{}.
\newblock \showarticletitle{Statistical model checking for cyber-physical
  systems}. In \bibinfo{booktitle}{\emph{International Symposium on Automated
  Technology for Verification and Analysis}}. Springer, \bibinfo{pages}{1--12}.
\newblock


\bibitem[\protect\citeauthoryear{Corbett-Davies, Pierson, Feller, Goel, and
  Huq}{Corbett-Davies et~al\mbox{.}}{2017}]%
        {corbett2017algorithmic}
\bibfield{author}{\bibinfo{person}{Sam Corbett-Davies}, \bibinfo{person}{Emma
  Pierson}, \bibinfo{person}{Avi Feller}, \bibinfo{person}{Sharad Goel}, {and}
  \bibinfo{person}{Aziz Huq}.} \bibinfo{year}{2017}\natexlab{}.
\newblock \showarticletitle{Algorithmic decision making and the cost of
  fairness}. In \bibinfo{booktitle}{\emph{Proceedings of the 23rd ACM SIGKDD
  International Conference on Knowledge Discovery and Data Mining}}. ACM,
  \bibinfo{pages}{797--806}.
\newblock


\bibitem[\protect\citeauthoryear{Datta, Sen, and Zick}{Datta
  et~al\mbox{.}}{2017}]%
        {datta2017algorithmic}
\bibfield{author}{\bibinfo{person}{Anupam Datta}, \bibinfo{person}{Shayak Sen},
  {and} \bibinfo{person}{Yair Zick}.} \bibinfo{year}{2017}\natexlab{}.
\newblock \showarticletitle{Algorithmic transparency via quantitative input
  influence}.
\newblock In \bibinfo{booktitle}{\emph{Transparent Data Mining for Big and
  Small Data}}. \bibinfo{publisher}{Springer}, \bibinfo{pages}{71--94}.
\newblock


\bibitem[\protect\citeauthoryear{Dwork, Hardt, Pitassi, Reingold, and
  Zemel}{Dwork et~al\mbox{.}}{2012}]%
        {dwork2012fairness}
\bibfield{author}{\bibinfo{person}{Cynthia Dwork}, \bibinfo{person}{Moritz
  Hardt}, \bibinfo{person}{Toniann Pitassi}, \bibinfo{person}{Omer Reingold},
  {and} \bibinfo{person}{Richard Zemel}.} \bibinfo{year}{2012}\natexlab{}.
\newblock \showarticletitle{Fairness through awareness}. In
  \bibinfo{booktitle}{\emph{Proceedings of the 3rd innovations in theoretical
  computer science conference}}. ACM, \bibinfo{pages}{214--226}.
\newblock


\bibitem[\protect\citeauthoryear{Dwork, Immorlica, Kalai, and Leiserson}{Dwork
  et~al\mbox{.}}{2018}]%
        {dwork2018decoupled}
\bibfield{author}{\bibinfo{person}{Cynthia Dwork}, \bibinfo{person}{Nicole
  Immorlica}, \bibinfo{person}{Adam~Tauman Kalai}, {and}
  \bibinfo{person}{Mark~DM Leiserson}.} \bibinfo{year}{2018}\natexlab{}.
\newblock \showarticletitle{Decoupled Classifiers for Group-Fair and Efficient
  Machine Learning}. In \bibinfo{booktitle}{\emph{Conference on Fairness,
  Accountability and Transparency}}. \bibinfo{pages}{119--133}.
\newblock


\bibitem[\protect\citeauthoryear{Esteva, Kuprel, Novoa, Ko, Swetter, Blau, and
  Thrun}{Esteva et~al\mbox{.}}{2017}]%
        {esteva2017dermatologist}
\bibfield{author}{\bibinfo{person}{Andre Esteva}, \bibinfo{person}{Brett
  Kuprel}, \bibinfo{person}{Roberto~A Novoa}, \bibinfo{person}{Justin Ko},
  \bibinfo{person}{Susan~M Swetter}, \bibinfo{person}{Helen~M Blau}, {and}
  \bibinfo{person}{Sebastian Thrun}.} \bibinfo{year}{2017}\natexlab{}.
\newblock \showarticletitle{Dermatologist-level classification of skin cancer
  with deep neural networks}.
\newblock \bibinfo{journal}{\emph{Nature}} \bibinfo{volume}{542},
  \bibinfo{number}{7639} (\bibinfo{year}{2017}), \bibinfo{pages}{115}.
\newblock


\bibitem[\protect\citeauthoryear{Feldman, Friedler, Moeller, Scheidegger, and
  Venkatasubramanian}{Feldman et~al\mbox{.}}{2015}]%
        {feldman2015certifying}
\bibfield{author}{\bibinfo{person}{Michael Feldman}, \bibinfo{person}{Sorelle~A
  Friedler}, \bibinfo{person}{John Moeller}, \bibinfo{person}{Carlos
  Scheidegger}, {and} \bibinfo{person}{Suresh Venkatasubramanian}.}
  \bibinfo{year}{2015}\natexlab{}.
\newblock \showarticletitle{Certifying and removing disparate impact}. In
  \bibinfo{booktitle}{\emph{Proceedings of the 21th ACM SIGKDD International
  Conference on Knowledge Discovery and Data Mining}}. ACM,
  \bibinfo{pages}{259--268}.
\newblock


\bibitem[\protect\citeauthoryear{Filieri, P{\u{a}}s{\u{a}}reanu, and
  Visser}{Filieri et~al\mbox{.}}{2013}]%
        {filieri2013reliability}
\bibfield{author}{\bibinfo{person}{Antonio Filieri}, \bibinfo{person}{Corina~S
  P{\u{a}}s{\u{a}}reanu}, {and} \bibinfo{person}{Willem Visser}.}
  \bibinfo{year}{2013}\natexlab{}.
\newblock \showarticletitle{Reliability analysis in symbolic pathfinder}. In
  \bibinfo{booktitle}{\emph{Proceedings of the 2013 International Conference on
  Software Engineering}}. IEEE Press, \bibinfo{pages}{622--631}.
\newblock


\bibitem[\protect\citeauthoryear{Fish, Kun, and Lelkes}{Fish
  et~al\mbox{.}}{2016}]%
        {fish2016confidence}
\bibfield{author}{\bibinfo{person}{Benjamin Fish}, \bibinfo{person}{Jeremy
  Kun}, {and} \bibinfo{person}{{\'A}d{\'a}m~D Lelkes}.}
  \bibinfo{year}{2016}\natexlab{}.
\newblock \showarticletitle{A confidence-based approach for balancing fairness
  and accuracy}. In \bibinfo{booktitle}{\emph{Proceedings of the 2016 SIAM
  International Conference on Data Mining}}. SIAM, \bibinfo{pages}{144--152}.
\newblock


\bibitem[\protect\citeauthoryear{Galhotra, Brun, and Meliou}{Galhotra
  et~al\mbox{.}}{2017}]%
        {galhotra2017fairness}
\bibfield{author}{\bibinfo{person}{Sainyam Galhotra}, \bibinfo{person}{Yuriy
  Brun}, {and} \bibinfo{person}{Alexandra Meliou}.}
  \bibinfo{year}{2017}\natexlab{}.
\newblock \showarticletitle{Fairness testing: testing software for
  discrimination}. In \bibinfo{booktitle}{\emph{Proceedings of the 2017 11th
  Joint Meeting on Foundations of Software Engineering}}. ACM,
  \bibinfo{pages}{498--510}.
\newblock


\bibitem[\protect\citeauthoryear{Gehr, Mirman, Drachsler-Cohen, Tsankov,
  Chaudhuri, and Vechev}{Gehr et~al\mbox{.}}{2018}]%
        {gehr2018safety}
\bibfield{author}{\bibinfo{person}{Timon Gehr}, \bibinfo{person}{Matthew
  Mirman}, \bibinfo{person}{Dana Drachsler-Cohen}, \bibinfo{person}{Petar
  Tsankov}, \bibinfo{person}{Swarat Chaudhuri}, {and} \bibinfo{person}{Martin
  Vechev}.} \bibinfo{year}{2018}\natexlab{}.
\newblock \showarticletitle{AI2: Safety and Robustness Certification of Neural
  Networks with Abstract Interpretation}. In \bibinfo{booktitle}{\emph{IEEE
  Symposium on Security and Privacy}}.
\newblock


\bibitem[\protect\citeauthoryear{Gehr, Misailovic, and Vechev}{Gehr
  et~al\mbox{.}}{2016}]%
        {gehr2016psi}
\bibfield{author}{\bibinfo{person}{Timon Gehr}, \bibinfo{person}{Sasa
  Misailovic}, {and} \bibinfo{person}{Martin Vechev}.}
  \bibinfo{year}{2016}\natexlab{}.
\newblock \showarticletitle{Psi: Exact symbolic inference for probabilistic
  programs}. In \bibinfo{booktitle}{\emph{CAV}}.
\newblock


\bibitem[\protect\citeauthoryear{Geldenhuys, Dwyer, and Visser}{Geldenhuys
  et~al\mbox{.}}{2012}]%
        {geldenhuys2012probabilistic}
\bibfield{author}{\bibinfo{person}{Jaco Geldenhuys}, \bibinfo{person}{Matthew~B
  Dwyer}, {and} \bibinfo{person}{Willem Visser}.}
  \bibinfo{year}{2012}\natexlab{}.
\newblock \showarticletitle{Probabilistic symbolic execution}. In
  \bibinfo{booktitle}{\emph{Proceedings of the 2012 International Symposium on
  Software Testing and Analysis}}. ACM, \bibinfo{pages}{166--176}.
\newblock


\bibitem[\protect\citeauthoryear{Goodfellow, Shlens, and Szegedy}{Goodfellow
  et~al\mbox{.}}{2014}]%
        {goodfellow2014explaining}
\bibfield{author}{\bibinfo{person}{Ian~J Goodfellow}, \bibinfo{person}{Jonathon
  Shlens}, {and} \bibinfo{person}{Christian Szegedy}.}
  \bibinfo{year}{2014}\natexlab{}.
\newblock \showarticletitle{Explaining and harnessing adversarial examples}. In
  \bibinfo{booktitle}{\emph{ICLR}}.
\newblock


\bibitem[\protect\citeauthoryear{Google}{Google}{2018}]%
        {drawclassify}
\bibfield{author}{\bibinfo{person}{Google}.} \bibinfo{year}{2018}\natexlab{}.
\newblock \bibinfo{title}{Recurrent Neural Networks for Drawing
  Classification}.
\newblock
  \bibinfo{howpublished}{\url{https://www.tensorflow.org/versions/master/tutorials/recurrent_quickdraw}}.
\newblock
\newblock
\shownote{Accessed: 2018-04-15.}


\bibitem[\protect\citeauthoryear{Gordon, Henzinger, Nori, and Rajamani}{Gordon
  et~al\mbox{.}}{2014}]%
        {gordon2014probabilistic}
\bibfield{author}{\bibinfo{person}{Andrew~D Gordon}, \bibinfo{person}{Thomas~A
  Henzinger}, \bibinfo{person}{Aditya~V Nori}, {and} \bibinfo{person}{Sriram~K
  Rajamani}.} \bibinfo{year}{2014}\natexlab{}.
\newblock \showarticletitle{Probabilistic programming}. In
  \bibinfo{booktitle}{\emph{Proceedings of the on Future of Software
  Engineering}}. ACM, \bibinfo{pages}{167--181}.
\newblock


\bibitem[\protect\citeauthoryear{Grosu and Smolka}{Grosu and Smolka}{2005}]%
        {grosu2005monte}
\bibfield{author}{\bibinfo{person}{Radu Grosu} {and} \bibinfo{person}{Scott~A
  Smolka}.} \bibinfo{year}{2005}\natexlab{}.
\newblock \showarticletitle{Monte carlo model checking}. In
  \bibinfo{booktitle}{\emph{International Conference on Tools and Algorithms
  for the Construction and Analysis of Systems}}. Springer,
  \bibinfo{pages}{271--286}.
\newblock


\bibitem[\protect\citeauthoryear{Ha and Eck}{Ha and Eck}{2017}]%
        {ha2017neural}
\bibfield{author}{\bibinfo{person}{David Ha} {and} \bibinfo{person}{Douglas
  Eck}.} \bibinfo{year}{2017}\natexlab{}.
\newblock \showarticletitle{A neural representation of sketch drawings}.
\newblock \bibinfo{journal}{\emph{arXiv preprint arXiv:1704.03477}}
  (\bibinfo{year}{2017}).
\newblock


\bibitem[\protect\citeauthoryear{Hajian and Domingo-Ferrer}{Hajian and
  Domingo-Ferrer}{2013}]%
        {hajian2013methodology}
\bibfield{author}{\bibinfo{person}{Sara Hajian} {and} \bibinfo{person}{Josep
  Domingo-Ferrer}.} \bibinfo{year}{2013}\natexlab{}.
\newblock \showarticletitle{A methodology for direct and indirect
  discrimination prevention in data mining}.
\newblock \bibinfo{journal}{\emph{IEEE transactions on knowledge and data
  engineering}} \bibinfo{volume}{25}, \bibinfo{number}{7}
  (\bibinfo{year}{2013}), \bibinfo{pages}{1445--1459}.
\newblock


\bibitem[\protect\citeauthoryear{Hardt, Price, and Srebro}{Hardt
  et~al\mbox{.}}{2016}]%
        {hardt2016equality}
\bibfield{author}{\bibinfo{person}{Moritz Hardt}, \bibinfo{person}{Eric Price},
  {and} \bibinfo{person}{Nathan Srebro}.} \bibinfo{year}{2016}\natexlab{}.
\newblock \showarticletitle{Equality of opportunity in supervised learning}. In
  \bibinfo{booktitle}{\emph{NIPS}}. \bibinfo{pages}{3315--3323}.
\newblock


\bibitem[\protect\citeauthoryear{H{\'e}rault, Lassaigne, Magniette, and
  Peyronnet}{H{\'e}rault et~al\mbox{.}}{2004}]%
        {herault2004approximate}
\bibfield{author}{\bibinfo{person}{Thomas H{\'e}rault},
  \bibinfo{person}{Richard Lassaigne}, \bibinfo{person}{Fr{\'e}d{\'e}ric
  Magniette}, {and} \bibinfo{person}{Sylvain Peyronnet}.}
  \bibinfo{year}{2004}\natexlab{}.
\newblock \showarticletitle{Approximate probabilistic model checking}. In
  \bibinfo{booktitle}{\emph{International Workshop on Verification, Model
  Checking, and Abstract Interpretation}}. Springer, \bibinfo{pages}{73--84}.
\newblock


\bibitem[\protect\citeauthoryear{Herault, Lassaigne, and Peyronnet}{Herault
  et~al\mbox{.}}{2006}]%
        {herault2006apmc}
\bibfield{author}{\bibinfo{person}{Thomas Herault}, \bibinfo{person}{Richard
  Lassaigne}, {and} \bibinfo{person}{Sylvain Peyronnet}.}
  \bibinfo{year}{2006}\natexlab{}.
\newblock \showarticletitle{APMC 3.0: Approximate verification of discrete and
  continuous time Markov chains}. In \bibinfo{booktitle}{\emph{Quantitative
  Evaluation of Systems, 2006. QEST 2006. Third International Conference on}}.
  IEEE, \bibinfo{pages}{129--130}.
\newblock


\bibitem[\protect\citeauthoryear{Hoeffding}{Hoeffding}{1963}]%
        {hoeffding1963probability}
\bibfield{author}{\bibinfo{person}{Wassily Hoeffding}.}
  \bibinfo{year}{1963}\natexlab{}.
\newblock \showarticletitle{Probability inequalities for sums of bounded random
  variables}.
\newblock \bibinfo{journal}{\emph{Journal of the American statistical
  association}} \bibinfo{volume}{58}, \bibinfo{number}{301}
  (\bibinfo{year}{1963}), \bibinfo{pages}{13--30}.
\newblock


\bibitem[\protect\citeauthoryear{Huang, Kwiatkowska, Wang, and Wu}{Huang
  et~al\mbox{.}}{2017}]%
        {huang2017safety}
\bibfield{author}{\bibinfo{person}{Xiaowei Huang}, \bibinfo{person}{Marta
  Kwiatkowska}, \bibinfo{person}{Sen Wang}, {and} \bibinfo{person}{Min Wu}.}
  \bibinfo{year}{2017}\natexlab{}.
\newblock \showarticletitle{Safety verification of deep neural networks}. In
  \bibinfo{booktitle}{\emph{International Conference on Computer Aided
  Verification}}. Springer, \bibinfo{pages}{3--29}.
\newblock


\bibitem[\protect\citeauthoryear{Johari, Koomen, Pekelis, and Walsh}{Johari
  et~al\mbox{.}}{2017}]%
        {johari2017peeking}
\bibfield{author}{\bibinfo{person}{Ramesh Johari}, \bibinfo{person}{Pete
  Koomen}, \bibinfo{person}{Leonid Pekelis}, {and} \bibinfo{person}{David
  Walsh}.} \bibinfo{year}{2017}\natexlab{}.
\newblock \showarticletitle{Peeking at a/b tests: Why it matters, and what to
  do about it}. In \bibinfo{booktitle}{\emph{Proceedings of the 23rd ACM SIGKDD
  International Conference on Knowledge Discovery and Data Mining}}. ACM,
  \bibinfo{pages}{1517--1525}.
\newblock


\bibitem[\protect\citeauthoryear{Katz, Barrett, Dill, Julian, and
  Kochenderfer}{Katz et~al\mbox{.}}{2017}]%
        {katz2017reluplex}
\bibfield{author}{\bibinfo{person}{Guy Katz}, \bibinfo{person}{Clark Barrett},
  \bibinfo{person}{David~L Dill}, \bibinfo{person}{Kyle Julian}, {and}
  \bibinfo{person}{Mykel~J Kochenderfer}.} \bibinfo{year}{2017}\natexlab{}.
\newblock \showarticletitle{Reluplex: An efficient SMT solver for verifying
  deep neural networks}. In \bibinfo{booktitle}{\emph{International Conference
  on Computer Aided Verification}}. Springer, \bibinfo{pages}{97--117}.
\newblock


\bibitem[\protect\citeauthoryear{Kilbertus, Carulla, Parascandolo, Hardt,
  Janzing, and Sch{\"o}lkopf}{Kilbertus et~al\mbox{.}}{2017}]%
        {kilbertus2017avoiding}
\bibfield{author}{\bibinfo{person}{Niki Kilbertus},
  \bibinfo{person}{Mateo~Rojas Carulla}, \bibinfo{person}{Giambattista
  Parascandolo}, \bibinfo{person}{Moritz Hardt}, \bibinfo{person}{Dominik
  Janzing}, {and} \bibinfo{person}{Bernhard Sch{\"o}lkopf}.}
  \bibinfo{year}{2017}\natexlab{}.
\newblock \showarticletitle{Avoiding discrimination through causal reasoning}.
  In \bibinfo{booktitle}{\emph{Advances in Neural Information Processing
  Systems}}. \bibinfo{pages}{656--666}.
\newblock


\bibitem[\protect\citeauthoryear{Kleinberg, Mullainathan, and
  Raghavan}{Kleinberg et~al\mbox{.}}{2017}]%
        {kleinberg2016inherent}
\bibfield{author}{\bibinfo{person}{Jon Kleinberg}, \bibinfo{person}{Sendhil
  Mullainathan}, {and} \bibinfo{person}{Manish Raghavan}.}
  \bibinfo{year}{2017}\natexlab{}.
\newblock \showarticletitle{Inherent trade-offs in the fair determination of
  risk scores}. In \bibinfo{booktitle}{\emph{ITCS}}.
\newblock


\bibitem[\protect\citeauthoryear{Kusner, Loftus, Russell, and Silva}{Kusner
  et~al\mbox{.}}{2017}]%
        {kusner2017counterfactual}
\bibfield{author}{\bibinfo{person}{Matt~J Kusner}, \bibinfo{person}{Joshua
  Loftus}, \bibinfo{person}{Chris Russell}, {and} \bibinfo{person}{Ricardo
  Silva}.} \bibinfo{year}{2017}\natexlab{}.
\newblock \showarticletitle{Counterfactual fairness}. In
  \bibinfo{booktitle}{\emph{Advances in Neural Information Processing
  Systems}}. \bibinfo{pages}{4069--4079}.
\newblock


\bibitem[\protect\citeauthoryear{Kwiatkowska, Norman, and Parker}{Kwiatkowska
  et~al\mbox{.}}{2002}]%
        {kwiatkowska2002prism}
\bibfield{author}{\bibinfo{person}{Marta Kwiatkowska}, \bibinfo{person}{Gethin
  Norman}, {and} \bibinfo{person}{David Parker}.}
  \bibinfo{year}{2002}\natexlab{}.
\newblock \showarticletitle{PRISM: Probabilistic symbolic model checker}. In
  \bibinfo{booktitle}{\emph{International Conference on Modelling Techniques
  and Tools for Computer Performance Evaluation}}. Springer,
  \bibinfo{pages}{200--204}.
\newblock


\bibitem[\protect\citeauthoryear{Lakkaraju, Kleinberg, Leskovec, Ludwig, and
  Mullainathan}{Lakkaraju et~al\mbox{.}}{2017}]%
        {lakkaraju2017selective}
\bibfield{author}{\bibinfo{person}{Himabindu Lakkaraju}, \bibinfo{person}{Jon
  Kleinberg}, \bibinfo{person}{Jure Leskovec}, \bibinfo{person}{Jens Ludwig},
  {and} \bibinfo{person}{Sendhil Mullainathan}.}
  \bibinfo{year}{2017}\natexlab{}.
\newblock \showarticletitle{The Selective Labels Problem: Evaluating
  Algorithmic Predictions in the Presence of Unobservables}. In
  \bibinfo{booktitle}{\emph{KDD}}.
\newblock


\bibitem[\protect\citeauthoryear{Lawrence}{Lawrence}{1991}]%
        {lawrence1991polytope}
\bibfield{author}{\bibinfo{person}{Jim Lawrence}.}
  \bibinfo{year}{1991}\natexlab{}.
\newblock \showarticletitle{Polytope volume computation}.
\newblock \bibinfo{journal}{\emph{Math. Comp.}} \bibinfo{volume}{57},
  \bibinfo{number}{195} (\bibinfo{year}{1991}), \bibinfo{pages}{259--271}.
\newblock


\bibitem[\protect\citeauthoryear{Legay, Delahaye, and Bensalem}{Legay
  et~al\mbox{.}}{2010}]%
        {legay2010statistical}
\bibfield{author}{\bibinfo{person}{Axel Legay}, \bibinfo{person}{Beno{\^\i}t
  Delahaye}, {and} \bibinfo{person}{Saddek Bensalem}.}
  \bibinfo{year}{2010}\natexlab{}.
\newblock \showarticletitle{Statistical model checking: An overview}. In
  \bibinfo{booktitle}{\emph{International conference on runtime verification}}.
\newblock


\bibitem[\protect\citeauthoryear{Monniaux}{Monniaux}{2000}]%
        {monniaux2000abstract}
\bibfield{author}{\bibinfo{person}{David Monniaux}.}
  \bibinfo{year}{2000}\natexlab{}.
\newblock \showarticletitle{Abstract interpretation of probabilistic
  semantics}. In \bibinfo{booktitle}{\emph{International Static Analysis
  Symposium}}. Springer, \bibinfo{pages}{322--339}.
\newblock


\bibitem[\protect\citeauthoryear{Monniaux}{Monniaux}{2001a}]%
        {monniaux2001abstract}
\bibfield{author}{\bibinfo{person}{David Monniaux}.}
  \bibinfo{year}{2001}\natexlab{a}.
\newblock \showarticletitle{An abstract Monte-Carlo method for the analysis of
  probabilistic programs}. In \bibinfo{booktitle}{\emph{ACM SIGPLAN Notices}},
  Vol.~\bibinfo{volume}{36}. ACM, \bibinfo{pages}{93--101}.
\newblock


\bibitem[\protect\citeauthoryear{Monniaux}{Monniaux}{2001b}]%
        {monniaux2001backwards}
\bibfield{author}{\bibinfo{person}{David Monniaux}.}
  \bibinfo{year}{2001}\natexlab{b}.
\newblock \showarticletitle{Backwards abstract interpretation of probabilistic
  programs}. In \bibinfo{booktitle}{\emph{European Symposium on Programming}}.
  Springer, \bibinfo{pages}{367--382}.
\newblock


\bibitem[\protect\citeauthoryear{Nabi and Shpitser}{Nabi and Shpitser}{2018}]%
        {nabi2018fair}
\bibfield{author}{\bibinfo{person}{Razieh Nabi} {and} \bibinfo{person}{Ilya
  Shpitser}.} \bibinfo{year}{2018}\natexlab{}.
\newblock \showarticletitle{Fair inference on outcomes}. In
  \bibinfo{booktitle}{\emph{AAAI}}, Vol.~\bibinfo{volume}{2018}.
\newblock


\bibitem[\protect\citeauthoryear{Pedreshi, Ruggieri, and Turini}{Pedreshi
  et~al\mbox{.}}{2008}]%
        {pedreshi2008discrimination}
\bibfield{author}{\bibinfo{person}{Dino Pedreshi}, \bibinfo{person}{Salvatore
  Ruggieri}, {and} \bibinfo{person}{Franco Turini}.}
  \bibinfo{year}{2008}\natexlab{}.
\newblock \showarticletitle{Discrimination-aware data mining}. In
  \bibinfo{booktitle}{\emph{Proceedings of the 14th ACM SIGKDD international
  conference on Knowledge discovery and data mining}}. ACM,
  \bibinfo{pages}{560--568}.
\newblock


\bibitem[\protect\citeauthoryear{Picchi}{Picchi}{2019}]%
        {powerball}
\bibfield{author}{\bibinfo{person}{Aimee Picchi}.}
  \bibinfo{year}{2019}\natexlab{}.
\newblock \showarticletitle{Odds of winning \$1 billion Mega Millions and
  Powerball: 1 in 88 quadrillion}.
\newblock \bibinfo{journal}{\emph{CBS News}} (\bibinfo{year}{2019}).
\newblock
\urldef\tempurl%
\url{https://www.cbsnews.com/news/odds-of-winning-1-billion-mega-millions-and-powerball-1-in-88-quadrillion}
\showURL{%
\tempurl}


\bibitem[\protect\citeauthoryear{Raghunathan, Steinhardt, and
  Liang}{Raghunathan et~al\mbox{.}}{2018}]%
        {raghunathan2018certified}
\bibfield{author}{\bibinfo{person}{Aditi Raghunathan}, \bibinfo{person}{Jacob
  Steinhardt}, {and} \bibinfo{person}{Percy Liang}.}
  \bibinfo{year}{2018}\natexlab{}.
\newblock \showarticletitle{Certified defenses against adversarial examples}.
  In \bibinfo{booktitle}{\emph{ICLR}}.
\newblock


\bibitem[\protect\citeauthoryear{Sampson, Panchekha, Mytkowicz, McKinley,
  Grossman, and Ceze}{Sampson et~al\mbox{.}}{2014}]%
        {sampson2014expressing}
\bibfield{author}{\bibinfo{person}{Adrian Sampson}, \bibinfo{person}{Pavel
  Panchekha}, \bibinfo{person}{Todd Mytkowicz}, \bibinfo{person}{Kathryn~S
  McKinley}, \bibinfo{person}{Dan Grossman}, {and} \bibinfo{person}{Luis
  Ceze}.} \bibinfo{year}{2014}\natexlab{}.
\newblock \showarticletitle{Expressing and verifying probabilistic assertions}.
  In \bibinfo{booktitle}{\emph{PLDI}}.
\newblock


\bibitem[\protect\citeauthoryear{Sankaranarayanan, Chakarov, and
  Gulwani}{Sankaranarayanan et~al\mbox{.}}{2013}]%
        {sankaranarayanan2013static}
\bibfield{author}{\bibinfo{person}{Sriram Sankaranarayanan},
  \bibinfo{person}{Aleksandar Chakarov}, {and} \bibinfo{person}{Sumit
  Gulwani}.} \bibinfo{year}{2013}\natexlab{}.
\newblock \showarticletitle{Static analysis for probabilistic programs:
  inferring whole program properties from finitely many paths}. In
  \bibinfo{booktitle}{\emph{PLDI}}. \bibinfo{pages}{447--458}.
\newblock


\bibitem[\protect\citeauthoryear{Sen, Viswanathan, and Agha}{Sen
  et~al\mbox{.}}{2004}]%
        {sen2004statistical}
\bibfield{author}{\bibinfo{person}{Koushik Sen}, \bibinfo{person}{Mahesh
  Viswanathan}, {and} \bibinfo{person}{Gul Agha}.}
  \bibinfo{year}{2004}\natexlab{}.
\newblock \showarticletitle{Statistical model checking of black-box
  probabilistic systems}. In \bibinfo{booktitle}{\emph{International Conference
  on Computer Aided Verification}}. Springer, \bibinfo{pages}{202--215}.
\newblock


\bibitem[\protect\citeauthoryear{Sen, Viswanathan, and Agha}{Sen
  et~al\mbox{.}}{2005}]%
        {sen2005statistical}
\bibfield{author}{\bibinfo{person}{Koushik Sen}, \bibinfo{person}{Mahesh
  Viswanathan}, {and} \bibinfo{person}{Gul Agha}.}
  \bibinfo{year}{2005}\natexlab{}.
\newblock \showarticletitle{On statistical model checking of stochastic
  systems}. In \bibinfo{booktitle}{\emph{International Conference on Computer
  Aided Verification}}. Springer, \bibinfo{pages}{266--280}.
\newblock


\bibitem[\protect\citeauthoryear{Simon}{Simon}{2009}]%
        {simon2009hp}
\bibfield{author}{\bibinfo{person}{Mallory Simon}.}
  \bibinfo{year}{2009}\natexlab{}.
\newblock \bibinfo{title}{HP looking into claim webcams can't see black
  people}.
\newblock
\newblock
\urldef\tempurl%
\url{http://www.cnn.com/2009/TECH/12/22/hp.webcams/index.html}
\showURL{%
\tempurl}


\bibitem[\protect\citeauthoryear{Tjeng and Tedrake}{Tjeng and Tedrake}{2017}]%
        {tjeng2017verifying}
\bibfield{author}{\bibinfo{person}{Vincent Tjeng} {and} \bibinfo{person}{Russ
  Tedrake}.} \bibinfo{year}{2017}\natexlab{}.
\newblock \showarticletitle{Verifying Neural Networks with Mixed Integer
  Programming}.
\newblock \bibinfo{journal}{\emph{arXiv preprint arXiv:1711.07356}}
  (\bibinfo{year}{2017}).
\newblock


\bibitem[\protect\citeauthoryear{Valiant}{Valiant}{1979}]%
        {valiant1979complexity}
\bibfield{author}{\bibinfo{person}{Leslie~G Valiant}.}
  \bibinfo{year}{1979}\natexlab{}.
\newblock \showarticletitle{The complexity of computing the permanent}.
\newblock \bibinfo{journal}{\emph{Theoretical computer science}}
  \bibinfo{volume}{8}, \bibinfo{number}{2} (\bibinfo{year}{1979}),
  \bibinfo{pages}{189--201}.
\newblock


\bibitem[\protect\citeauthoryear{Wald}{Wald}{1945}]%
        {wald1945sequential}
\bibfield{author}{\bibinfo{person}{Abraham Wald}.}
  \bibinfo{year}{1945}\natexlab{}.
\newblock \showarticletitle{Sequential tests of statistical hypotheses}.
\newblock \bibinfo{journal}{\emph{The annals of mathematical statistics}}
  \bibinfo{volume}{16}, \bibinfo{number}{2} (\bibinfo{year}{1945}),
  \bibinfo{pages}{117--186}.
\newblock


\bibitem[\protect\citeauthoryear{Wen, Bastani, and Topcu}{Wen
  et~al\mbox{.}}{2019}]%
        {wen2019fairness}
\bibfield{author}{\bibinfo{person}{Min Wen}, \bibinfo{person}{Osbert Bastani},
  {and} \bibinfo{person}{Ufuk Topcu}.} \bibinfo{year}{2019}\natexlab{}.
\newblock \showarticletitle{Fairness with Dynamics}.
\newblock \bibinfo{journal}{\emph{arXiv preprint arXiv:1901.08568}}
  (\bibinfo{year}{2019}).
\newblock


\bibitem[\protect\citeauthoryear{Younes, Musliner, et~al\mbox{.}}{Younes
  et~al\mbox{.}}{2002}]%
        {younes2002probabilistic}
\bibfield{author}{\bibinfo{person}{H{\aa}kan~LS Younes},
  \bibinfo{person}{David~J Musliner}, {et~al\mbox{.}}}
  \bibinfo{year}{2002}\natexlab{}.
\newblock \showarticletitle{Probabilistic plan verification through acceptance
  sampling}. In \bibinfo{booktitle}{\emph{Proceedings of the AIPS-02 Workshop
  on Planning via Model Checking}}. Citeseer, \bibinfo{pages}{81--88}.
\newblock


\bibitem[\protect\citeauthoryear{Younes and Simmons}{Younes and
  Simmons}{2002}]%
        {younes2002probabilistic2}
\bibfield{author}{\bibinfo{person}{H{\aa}kan~LS Younes} {and}
  \bibinfo{person}{Reid~G Simmons}.} \bibinfo{year}{2002}\natexlab{}.
\newblock \showarticletitle{Probabilistic verification of discrete event
  systems using acceptance sampling}. In
  \bibinfo{booktitle}{\emph{International Conference on Computer Aided
  Verification}}. Springer, \bibinfo{pages}{223--235}.
\newblock


\bibitem[\protect\citeauthoryear{Younes and Simmons}{Younes and
  Simmons}{2006}]%
        {younes2006statistical}
\bibfield{author}{\bibinfo{person}{H{\aa}kan~LS Younes} {and}
  \bibinfo{person}{Reid~G Simmons}.} \bibinfo{year}{2006}\natexlab{}.
\newblock \showarticletitle{Statistical probabilistic model checking with a
  focus on time-bounded properties}.
\newblock \bibinfo{journal}{\emph{Information and Computation}}
  \bibinfo{volume}{204}, \bibinfo{number}{9} (\bibinfo{year}{2006}),
  \bibinfo{pages}{1368--1409}.
\newblock


\bibitem[\protect\citeauthoryear{Younes}{Younes}{2004}]%
        {younes2004verification}
\bibfield{author}{\bibinfo{person}{Hakan Lorens~Samir Younes}.}
  \bibinfo{year}{2004}\natexlab{}.
\newblock \emph{\bibinfo{title}{Verification and Planning for Stochastic
  Processes with Asynchronous Events}}.
\newblock \bibinfo{thesistype}{Ph.D. Dissertation}.
  \bibinfo{address}{Pittsburgh, PA, USA}.
\newblock


\bibitem[\protect\citeauthoryear{Zafar, Valera, Gomez~Rodriguez, and
  Gummadi}{Zafar et~al\mbox{.}}{2017}]%
        {zafar2017fairness}
\bibfield{author}{\bibinfo{person}{Muhammad~Bilal Zafar},
  \bibinfo{person}{Isabel Valera}, \bibinfo{person}{Manuel Gomez~Rodriguez},
  {and} \bibinfo{person}{Krishna~P Gummadi}.} \bibinfo{year}{2017}\natexlab{}.
\newblock \showarticletitle{Fairness beyond disparate treatment \& disparate
  impact: Learning classification without disparate mistreatment}. In
  \bibinfo{booktitle}{\emph{Proceedings of the 26th International Conference on
  World Wide Web}}. International World Wide Web Conferences Steering
  Committee, \bibinfo{pages}{1171--1180}.
\newblock


\bibitem[\protect\citeauthoryear{Zarsky}{Zarsky}{2014}]%
        {zarsky2014understanding}
\bibfield{author}{\bibinfo{person}{Tal~Z Zarsky}.}
  \bibinfo{year}{2014}\natexlab{}.
\newblock \showarticletitle{Understanding discrimination in the scored
  society}.
\newblock \bibinfo{journal}{\emph{Wash. L. Rev.}}  \bibinfo{volume}{89}
  (\bibinfo{year}{2014}), \bibinfo{pages}{1375}.
\newblock


\bibitem[\protect\citeauthoryear{Zemel, Wu, Swersky, Pitassi, and Dwork}{Zemel
  et~al\mbox{.}}{2013}]%
        {zemel2013learning}
\bibfield{author}{\bibinfo{person}{Rich Zemel}, \bibinfo{person}{Yu Wu},
  \bibinfo{person}{Kevin Swersky}, \bibinfo{person}{Toni Pitassi}, {and}
  \bibinfo{person}{Cynthia Dwork}.} \bibinfo{year}{2013}\natexlab{}.
\newblock \showarticletitle{Learning fair representations}. In
  \bibinfo{booktitle}{\emph{International Conference on Machine Learning}}.
  \bibinfo{pages}{325--333}.
\newblock


\bibitem[\protect\citeauthoryear{Zhao, Zhou, Sabharwal, and Ermon}{Zhao
  et~al\mbox{.}}{2016}]%
        {zhao2016adaptive}
\bibfield{author}{\bibinfo{person}{Shengjia Zhao}, \bibinfo{person}{Enze Zhou},
  \bibinfo{person}{Ashish Sabharwal}, {and} \bibinfo{person}{Stefano Ermon}.}
  \bibinfo{year}{2016}\natexlab{}.
\newblock \showarticletitle{Adaptive Concentration Inequalities for Sequential
  Decision Problems}. In \bibinfo{booktitle}{\emph{NIPS}}.
  \bibinfo{pages}{1343--1351}.
\newblock


\end{thebibliography}
